\documentclass{article}
\usepackage[english]{babel}

\usepackage[letterpaper,top=2cm,bottom=2cm,left=2.5cm,right=2.5cm,marginparwidth=1.65cm]{geometry}
\usepackage{amsmath}
\usepackage{authblk}
\usepackage{multirow}
\usepackage{ctable}
\usepackage{graphicx}
\usepackage{subfig}
\usepackage{float}
\usepackage[linesnumbered,ruled,vlined]{algorithm2e}
\usepackage{marginnote}
\SetKwInput{KwInput}{Input}                
\SetKwInput{KwOutput}{Output}  

\SetCommentSty{mycommfont}
\usepackage[noend]{algpseudocode}

\algnewcommand{\algorithmicand}{\textbf{ and }}
\algnewcommand{\algorithmicor}{\textbf{ or }}
\algnewcommand{\OR}{\algorithmicor}
\algnewcommand{\AND}{\algorithmicand}
\algnewcommand{\var}{\texttt}

\usepackage[colorlinks=true, allcolors=blue]{hyperref}
\usepackage{amsthm}
\usepackage{amssymb}
\newtheorem{definition}{Definition}
\newtheorem{proposition}{Proposition}

\title{Approximately Optimal Search on a Higher-dimensional Sliding Puzzle}
\author[a,b,*]{Nono SC Merleau}
\author[a,b,*]{Miguel O'Malley}
\author[a,b]{Érika Roldán}
\author[a,b,c]{Sayan Mukherjee}
\affil[a]{Center for Scalable Data Analytics and Artificial Intelligence and Department of Computer Science, Universit\"{a}t Leipzig}
\affil[b]{Max Planck Institute for Mathematics in the Sciences}
\affil[c]{Duke University, Departments of Statistical Science, Mathematics, Computer Science, and Biostatistics \& Bioinformatics}
\affil[*]{Corresponding authors: csaha@aims.edu.gh/miguelomalley9@gmail.com}
\begin{document}
\maketitle

\begin{abstract}

Higher-dimensional sliding puzzles are constructed on the vertices of a $d$-dimensional hypercube, where $2^d-l$ vertices are distinctly coloured. Rings with the same colours are initially set randomly on the vertices of the hypercube. The goal of the puzzle is to move each of the $2^d-l$ rings to pre-defined target vertices on the cube. In this setting, the $k$-rule constraint represents a generalisation of edge collision for the movement of colours between vertices, allowing movement only when a hypercube face of dimension $k$ containing a ring is completely free of other rings.
Starting from an initial configuration, what is the minimum number of moves needed to make ring colours match the vertex colours? An algorithm that provides us with such a number is called God's algorithm. When such an algorithm exists, it does not have a polynomial time complexity, at least in the case of the 15-puzzle corresponding to $k=1$ in the cubical puzzle. This paper presents a comprehensive computational study of different scenarios of the higher-dimensional puzzle. A benchmark of three computational techniques, an exact algorithm (the A* search) and two approximately optimal search techniques (an evolutionary algorithm (EA) and reinforcement learning (RL)) is presented in this work. The experiments show that all three methods can successfully solve the puzzle of dimension three for different face dimensions and across various difficulty levels. When the dimension increases, the A* search fails, and RL and EA methods can still provide a generally acceptable solution, i.e. a distribution of a number of moves with a median value of less than $30$. Overall, the EA method consistently requires less computational time, while failing in most cases to minimise the number of moves for the puzzle dimensions $d=4$ and $d=5$. 

\end{abstract}

\section{Introduction}

Many engineering problems of scientific importance can be related to the general problem of finding the shortest path through a graph. When considering a game setting, the graph vertices are game configurations, and an edge exists between two vertices if a legal move (depending on the rules of the puzzle) allows one configuration to be obtained from another. Like finding the shortest path through a graph, solving a sliding puzzle involves finding the minimum number of moves (or slides) to get the target configuration from a given starting configuration. The $15$-puzzle is a classic example often employed to demonstrate the potential for some starting position configurations to yield unsolvable results (See Figure \autoref{fig:start} and Figure \autoref{fig:target} for illustration). A study in more depth of the $15$-puzzle for a more general graph has previously been done, and its solvability has been understood \cite{wilson1974graph}, but it still needs a general rule, i.e. for face dimension $k\geq1$. 

Inspired by the $15$-puzzle, Alpert \cite{alpert2019hexagons} examines square and hexagonal sliding puzzles in greater generality, to establish broader properties of their configuration spaces. Rold{\'a}n and Karpman \cite{karpman2022parity} extend their examination to establish parity properties of hexagonal sliding puzzles, identifying a broader analogue of the parity property present in the configuration space of the $15$-puzzle for the hexagonal case. Most recently, Rold{\'a}n et al. \cite{beyer2023higher} establish a generalisation of the sliding puzzle to the cubical case, establishing strong conditions for the solvability of cubical puzzles and characterising the connectivity regime of $k$-mobile configurations (those puzzle configurations with no stuck vertices) as either one or two large components. However, the computation of the diameter of the puzzle graph (the greatest integer representing a minimal number of moves from one configuration to another) is known to be NP-hard in the case of the $15$-puzzle \cite{goldreich}.

In the absence of theoretical precision, we implement methods from A* search and evolutionary algorithms (EA) to reinforcement learning (RL) to provide an appropriate estimate of the diameter for various cubical sliding puzzles. We further seek to compare the speed and accuracy of these methods to each other and human performance in an online version of this game. A straightforward and classic heuristic approach that guarantees the shortest path through graphs is the A* search algorithm. Two components are essential to the A* algorithm: an evaluation function (or the chosen heuristic) and a successor operator. As stated in the earlier work of Peter E. Hart \cite{hart1968formal}, A* search guarantees the optimal solution if the evaluation function chosen is admissible and monotonic. First, we design an A* algorithm adapted to the high-dimensional cubical puzzle and discuss its heuristic choice. Next, we design approaches using EA and RL.
\begin{figure*}[!t]
  \centering
  \subfloat[\textbf{15-puzzle: Starting configuration.}]{\includegraphics[width=0.3\textwidth]{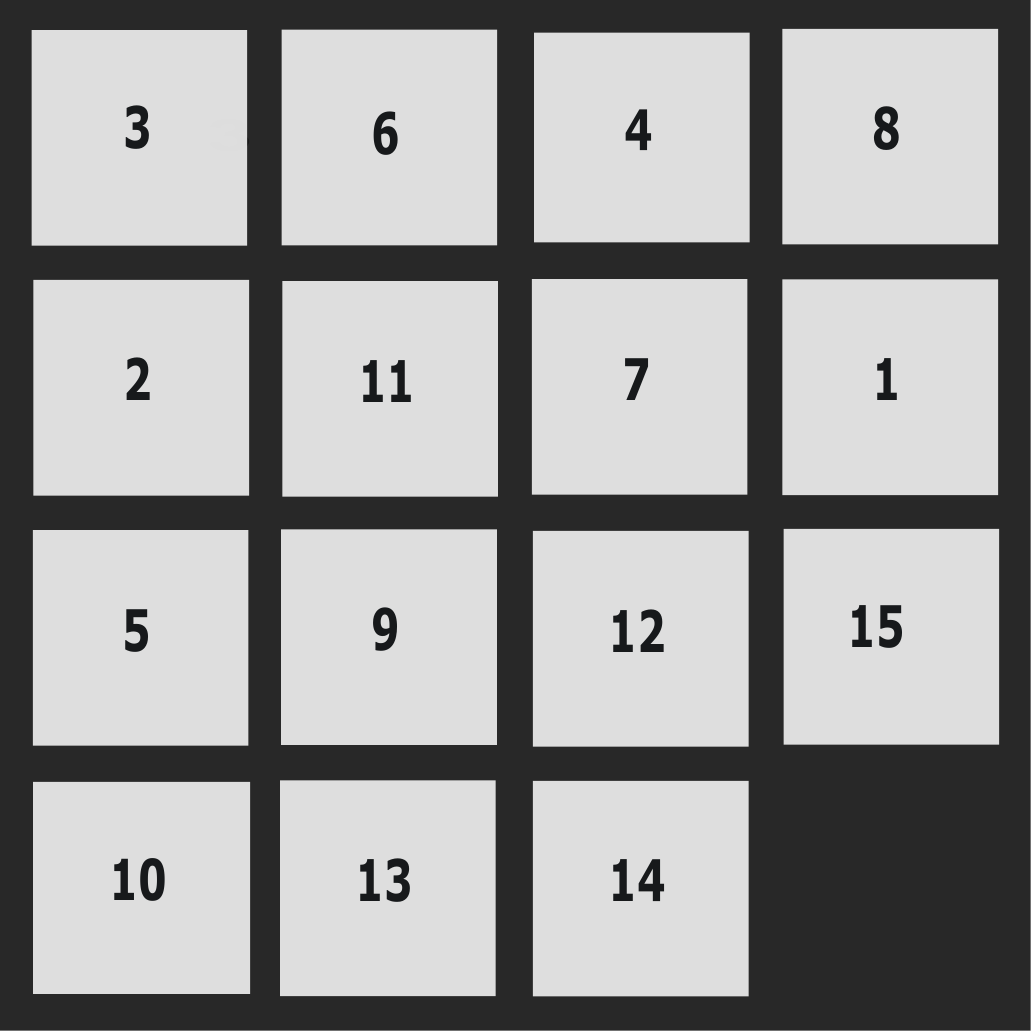}\label{fig:start}}
  \hfil
 \subfloat[\textbf{15-puzzle: Target configuration.}]{\includegraphics[width=0.3\textwidth]{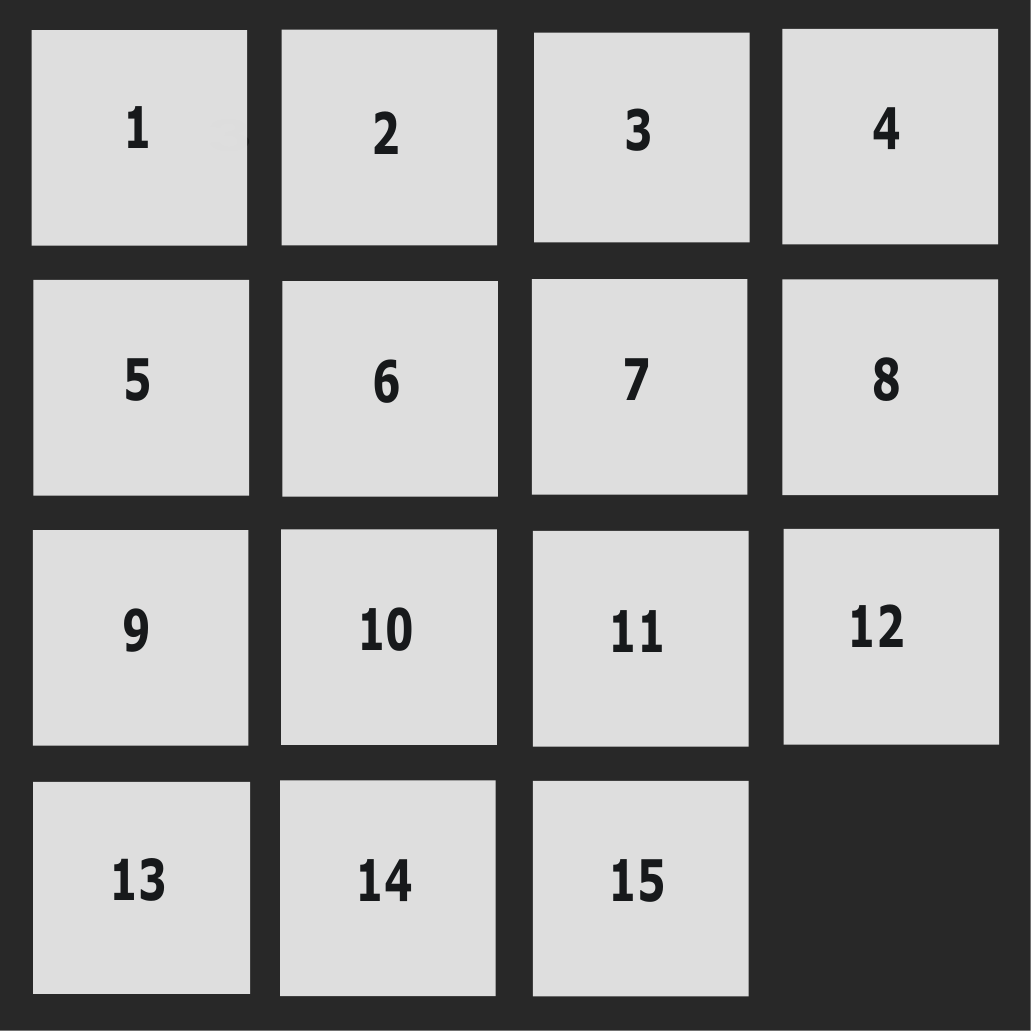}\label{fig:target}}

 \subfloat[\textbf{Cubical sliding puzzle: Starting configuration.}]{\includegraphics[width=0.35\textwidth]{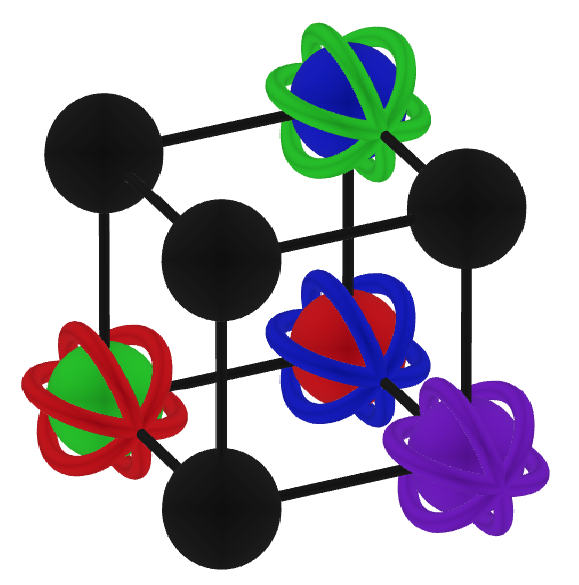}\label{fig:startHSP}}
\hfil
 \subfloat[\textbf{Cubical sliding puzzle: Target configuration.}]{\includegraphics[width=0.35\textwidth]{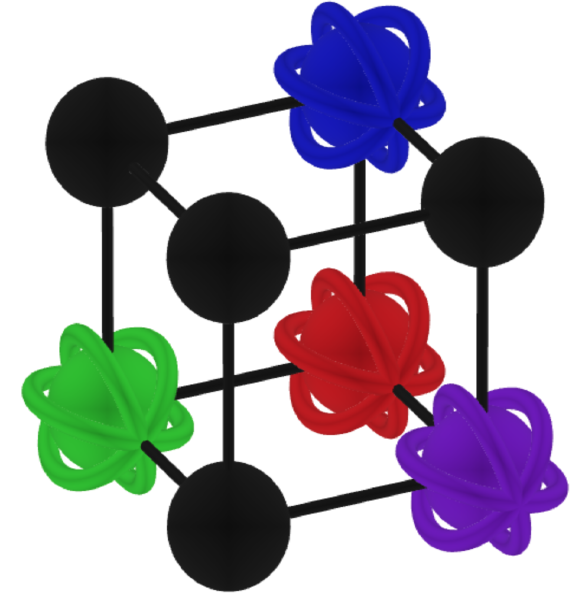}\label{fig:targetHSP}}
 \caption{\textbf{An illustration of the classical 15-puzzle and the high-dimensional puzzle of dimension $d=3$ for the difficulty level $0$}. 
On the top rows is a classic example of the $15$-puzzle. Starting from configuration (a), the puzzle consists of numbered square tiles which can be slid into a frame using an empty slot. The object uses the space to slide all tiles where they belong (b), the target configuration. 
The bottom rows show a high-dimensional sliding puzzle of dimension $d=3$. There are $l=4$ uncoloured vertices (i.e., in black) and $4$ coloured vertices with rings of the same colours. On the starting configuration (c), only one ring matches the vertex colour, the purple ring. In contrast, all the ring colours match the vertex colours on the target configuration (d).}
 \label{fig:sliding_puzzle}
\end{figure*}
The second class of optimisation techniques we explore in this work is the RL technique. RL's fundamental idea is that we learn by interacting with our environment \cite{andrew1999reinforcement}. This interaction produces information about cause, effect, and the consequences of actions; in RL, it is more about what we do to achieve a specific goal; in the case of the sliding puzzle, we may think of it as what moves to perform, to reach a particular target configuration. The main components of RL are often an agent, a set of policies, an environment and a reward function. In that regard, the agent interacts with a dynamic environment to achieve a particular goal, e.g., to maximise its reward. The technique is quite flexible and powerful in dealing with problems where an exact mathematical model of the environment is unknown, e.g., in large Markov Decision Processes (MDP). RL has found applications in many disciplines, such as game theory \cite{littman1994markov}, control theory \cite{lewis2013reinforcement}, operations research \cite{powell2007approximate}, information theory \cite{sutton2018reinforcement}, simulation-based optimisation \cite{gosavi2015simulation}, multi-agent systems \cite{busoniu2008comprehensive}, swarm intelligence \cite{yang2020nature}, and statistics \cite{murphy2003optimal}. 
Applying the RL method to any problem often requires modelling the problem as a sequential decision-making process, i.e., defining the MDP. The main ingredients of MDPs are usually the state space, the action space, the reward, and the transition functions \cite{bellman1957markovian}. In addition to the previously mentioned application fields, RL has already been applied to several classes of combinatoric optimisation problems \cite{mazyavkina2021reinforcement} such as the Travelling Salesman Problem (TSP) \cite{bello2016neural, nazari2018reinforcement},  Maximum Cut (Max-Cut) problem \cite{barrett2020exploratory}, Maximum Independent Set (MIS) \cite{li2018combinatorial, khalil2017learning}, Minimum Vertex Cover (MVC) \cite{yolcu2019learning}, and Bin Packing Problem (BPP) \cite{kundu2019deep, zhang2021attend2pack}. The authors \cite{mazyavkina2021reinforcement} provide an extensive overview of various RL applied to this class of problems. Designing an RL algorithm for the higher-dimensional sliding puzzle follows steps similar to those for the above-mentioned combinatorial problems. Our main challenge is to define a proper rewarding function that integrates most knowledge of the environment. The method section discusses the RL steps to solve the higher-dimensional puzzle and clarifies the reward function at each step.

Finally, we study the performance of EAs on the higher-dimensional sliding puzzle. Similar to RL methods, EAs are well-known meta-heuristic approaches to optimisation problems, especially when dealing with issues in which less information about the fitness landscapes is provided or when there is no exact algorithm in polynomial time for such matters. EA was inspired by evolutionary biological systems and since it was proposed by John Holland \cite{holland1992adaptation} in the early 1970s, it has emerged as a popular search heuristic. It has found application in many disciplines that deal with complex landscape optimisation problems, such as bioinformatics \cite{merleau2022arnaque, merleau2021simple, esmaili2014evolutionary, esmaili2015erd,wiese2008rnapredict}, engineering science \cite{mitchell1998introduction}, in solving complex variants of the Vehicle Routing Problem (VRP) \cite{jozefowiez2009evolutionary,zhao2009hybrid}. The picture in all these systems was to evolve a population of candidate solutions to a given situation (or target), using operators inspired by genetic variation and natural selection. 

Similarly, solving the higher-dimensional sliding puzzle using EA will involve a population of game configurations to a specific target configuration. As for most optimisation problems, designing an EA comprises three essential steps: genotype representation, mutation or/and crossover operators, and a selection procedure. Here, we describe a framework that allows us to apply EA to the higher-dimensional sliding puzzle to provide an estimate of the diameter distribution of a given puzzle. For comparison's sake, we also offer performance comparisons to an adapted RL algorithm and optimal A* search for completeness. Our comparative study shows that the A* algorithm can solve the puzzle of dimensions $3$ and $4$ for different face dimensions and across various difficulty levels except for level $4$ of the puzzle dimension $4$. Furthermore, EA and RL demonstrate a similar potential to solve the higher-dimensional puzzle across all levels and face dimensions but differ in the median number of moves and CPU time. While EA has a high median number of moves and low CPU time, RL solves the puzzles with a lower median number of moves but with higher CPU time.

This work is organised into three main sections. The first section briefly describes the game setting, and the second section provides the different algorithms, A*, RL, and EA. Finally, in the last section, we provide a comprehensive benchmark of the explored algorithms, including their performance on different puzzles and CPU time. 

\section{Game Setting}
The higher-dimensional sliding puzzle we study in this work is built on a $d$-dimensional cube (Each node of the cube is a binary string of length $d$, and two nodes are connected if they differ by a single bit) where $2^d-l$ randomly selected vertices are coloured (i.e., $l$ vertices are uncoloured or have the same colour). Then, there are precisely $2^d-l$ rings with the same colours initially set randomly on vertices. The $15$-puzzle can be seen as a version of this game played on a ($4\times4$)-grid. In that setting, rings block each other's movement simply by being in the way. However, topologically, this can be considered to represent the rule that movement is blocked when the $1$-simplex where the ring would move over is occupied by another ring. The higher-dimensional puzzle setting is a generalisation of this puzzle's setup as the $(d,k,\ell)$ scheme, where the puzzle is played on a $d$-dimensional cube with $2^d-\ell$ coloured vertices and rings, and where each move consists of moving one ring to a vertex which shared the same $k$-face so long as any other rings do not occupy that face \cite{wilson1974puzzles}.

More formally, for a given $d$ and $l$, let $\mathcal{C}[d,l]$ be the set of all configurations of $2^d-l$ labeled/coloured vertices on the vertices of the hypercube $Q^d$, and $\mathcal{L}_C =\{c_1,c_2,...,c_{2^d-l}\}$ the set of labels/colours of these vertices.

\begin{definition}
	Let $\mathcal{F}$ be a $k$-dimensional face of $Q^d$ and $r$ a ring on one of the vertices of $\mathcal{F}$. If the remaining vertices of $\mathcal{F}$ are unoccupied, then we say that $r$ is in a \textit{free $k$-state with respect to the face $\mathcal{F}$}. A ring that is on a free $k$-state with respect to  $\mathcal{F}$ can move (or slide) to occupy any of the vertices of $\mathcal{F}$. We call this move of $r$ a \textit{$k$-move on $\mathcal{F}$}.
    \label{Def:krule}
\end{definition}

Definition \ref{Def:krule} is fundamental for the essential components of the algorithms we study in this work, such as the mutation operator, successor operator and actions for the respective EA, A*, and RL algorithms. \autoref{fig:sliding_puzzle} shows an example of the sliding puzzle of dimension $d=3$ and difficulty level $0$. Starting from the starting configuration (see Figure \autoref{fig:startHSP}), the sliding puzzle consists of moving (or sliding) the rings to free vertices until all the ring colours match the vertex colours (i.e., the target configuration, Figure \autoref{fig:targetHSP}). The rings move based on the Definition \ref{Def:krule} on the faces of the $d-$dimensional cube, e.g. the purple ring can move to any of its adjacent vertices since the face occupied by the purple ring is free.


The difficulty level of the sliding puzzle depends directly on the number of uncoloured vertices $l$ and the dimensions of the cube and face. Throughout this paper, we consider for a fixed face dimension $k$ and number of colours $2^d -l$, puzzles of higher dimensions harder than the ones of lower dimensions. In addition to the previously mentioned parameters, the starting configuration is necessary to define the puzzle's difficulty level properly. \autoref{tab:game_parameters} summarises the main parameters of the sliding puzzle we will use throughout our benchmark analysis. 
\begin{table}[H]
    \centering
    \caption{Main parameters of the sliding puzzle}
    \begin{tabular}{c|p{6cm}|c}
         \toprule Parameter& Description & Constraints \\
         \midrule $d$& The dimension of the cube& $d\geq 1$\\
         $k$& The dimension of the faces& $k\leq d$ \\ 
         $l$ & The number of unlabelled (equally the number of uncoloured vertices) & $d\geq 1, 1\leq l \leq 2^d$\\ 
         $\mathcal{S}$ & Starting configuration & A list of tuples (ring position, colour)\\ 
         $\mathcal{T}$ & Target configuration & A list of tuples (ring position, colour)\\ 
        \bottomrule
    \end{tabular}
    \label{tab:game_parameters}
\end{table}
For the puzzles studied in the next sections, the parameters $(d,l)$ are fixed, and different puzzles are generated, each differing from the starting configuration. Then, for different starting configurations(or difficulty levels), we analyse the algorithm's performances for various face dimensions. The \autoref{appendix:data} provides a list of starting and target configurations for puzzles of dimensions $3, 4$, and $5$. 
\section{Computational approaches to the higher-dimensional sliding puzzle}
We study in this work three computational approaches to the higher-dimensional sliding puzzle: the $\text{A}^*$ search, the RL and the EA. This section presents an overview of each technique and a general description of their corresponding algorithms.
\subsection{A* search}
The sliding puzzle is defined on a hypercube $Q_d$, where each position of the rings leads to a unique game configuration. Both the target ($\mathcal{T}$) and the starting ($\mathcal{S}$) configurations are represented by a list of pairs (ring position and ring colour) of the same length ($2^d - l$). The goal is to find the minimum number of moves to get to $\mathcal{T}$, starting from $\mathcal{S}$. 

As with any puzzle seeking the shortest path between two points in the configuration space, slow, optimal methods exist that will surely return the shortest possible solution; perhaps the most well-known among these methods is A* search, implemented easily, albeit at great computational expense. In the case of A* search, each ring position and colour list is converted into a binary vector of dimension $d\times (2^d -l)$.
\begin{algorithm}[t!]
	\tcc{$Q^d$: a hypercube of dimension $d$ on which the sliding puzzle is defined\;
        $k \in \left \{ 1,\cdots, d-1\right\}$: the $k$ rule for the given puzzle\;
		$\mathcal{S}\subset Q^d$: the initial ordered unique configuration of rings on the hypercube $Q^d$\;
        $\mathcal{T}\subset Q^d$: the target ordered unique configuration of rings on the hypercube $Q^d$\;
        $g_{\mathcal{S}}(\mathcal{C})$: a function returning the number of moves required to reach configuration $\mathcal{C}\subset Q^d$ from $\mathcal{S}$\;
        $h_{\mathcal{T}}(\mathcal{C})$: a function returning a lower bound on the minimum number of moves required to reach $\mathcal{T}$ from $\mathcal{C}$\;
        $f_{\mathcal{S},\mathcal{T}}(\mathcal{C}) = g_{\mathcal{S}}(\mathcal{C})+f_{\mathcal{T}}(\mathcal{C})$: the sum of $h$ and $g$ on $\mathcal{C}$\;
        $p(\mathcal{C})$: a function returning the parent configuration of $\mathcal{C}$\;
        $\mathcal{N}$: the set of configurations directly reachable from $\mathcal{C}$\;
        $M = \left\{\mathcal{S},\mathcal{C}_1,\dots,\mathcal{C}_{n-1}, \mathcal{T}\right\}$: a list of configurations between $\mathcal{S}$ and $\mathcal{T}$;}
	\KwInput{$\mathcal{S}$, $\mathcal{T}$, $Q^d$, $g_{\mathcal{S}}$, $h_{\mathcal{S}}$, $f_{\mathcal{S},\mathcal{T}}$}
	\KwOutput{$M$}
        $\mathcal{C}^* \leftarrow\mathcal{S}$ \tcp*{The current configuration.}
	    $\mathcal{N} \leftarrow N(\mathcal{S})$ \tcp*{Get neighbors of starting configuration.}
	$O \leftarrow \{\}$ \tcp*{The open set of configurations.}
        $X \leftarrow \{\}$ \tcp*{It is defined to be the closed set of configurations.}\;
        \While{$\mathcal{C}^*\neq \mathcal{T}$}{
    	\For{$\mathcal{C} \in \mathcal{N}$}
                {\If{$\mathcal{C}\in O$} 
                        {\If{$g_{\mathcal{S}}(\mathcal{C}^*)+1<g_{\mathcal{S}}(\mathcal{C})$}{
                            $p(\mathcal{C})\leftarrow\mathcal{C}^*$\tcp*{This path is shorter, record it.}
                            $g_{\mathcal{S}}(\mathcal{C}) \leftarrow g_{\mathcal{S}}(\mathcal{C}^*)+1$ \tcp*{Adjust cost.}
                            }
                        }
                \If{$\mathcal{C}\notin O$ \AND $\mathcal{C}\notin X$}{
                    $O = O\cup\{\mathcal{C}\}$ \tcp*{This configuration is new, record it.}
                    $p(\mathcal{C}) \leftarrow \mathcal{C}^*$\;
                    $g_{\mathcal{S}}(\mathcal{C}) \leftarrow g_{\mathcal{S}}(\mathcal{C}^*)+1$ \tcp*{This configuration has taken one additional step to reach.}
                    } 
    		}
            $X = X\cup \{\mathcal{C}^*\}$\tcp*{Add current configuration to the closed list.}
            
            \For{$\mathcal{C}'\in O$}{ 
            \If{$f_{\mathcal{S},\mathcal{T}}(\mathcal{C})\leq f_{\mathcal{S},\mathcal{T}}(\mathcal{C}')$} {
             Add $\mathcal{C}'$ to $\mathcal{C}^*$ \tcp*{
           Select the configuration with least $f$ cost and set it to current.}
        
            }
        }
        $O \leftarrow O\setminus \mathcal{C}^*$ \tcp*{Remove $\mathcal{C}^*$ from the open set.}
        }
        $M \leftarrow \{\mathcal{S}\}$ \tcp{Initialise $M$.}
        \While{$p(\mathcal{C}^*)$ is not undefined}{
            $M \leftarrow M \cup \{p(\mathcal{C}^*) \}$ \tcp{Construct list of configurations by backtracking.}
            $\mathcal{C}^* \leftarrow p(\mathcal{C}^*)$.}
	\caption{A* search}\label{Algo:A*}
    
\end{algorithm}

The classical A* search algorithm seeks to find the minimum cost path of a subgraph $G_s$ starting from a game configuration $\mathcal{S}$ to a preferred goal configuration $\mathcal{T}$. Starting with a node $\mathcal{S}$, the algorithm generates some part of the sub-graph $G_s$ by repetitively applying a successor operator. 
The general search algorithm A* follows the steps \cite{hart1968formal}:
\begin{enumerate}
    \item Mark $\mathcal{S}$ ``open'' and evaluate $\mathcal{S}$ with respect to the evaluation function $f$.
    \item Select the open node $\mathcal{C}$ whose value of $f$ is smallest.
    \item Resolve ties arbitrarily, but always in favour of any node $\mathcal{C}$.
    \item  If $\mathcal{C} \equiv \mathcal{T}$, mark $\mathcal{C}$ ``closed'' and terminate the algorithm.
    \item Otherwise, mark $\mathcal{C}$ closed and apply the successor operator to $\mathcal{C}$. Calculate $f$ for each successor of $\mathcal{C}$ and mark each successor not already marked closed as open.
\end{enumerate}
A few specifications are needed to implement the A* method in the sliding puzzle. Two important concepts are to be defined: 
\subsubsection*{The evaluation function ($f$)} 
The central concept of the A* algorithm is to propose an evaluation function $f(\mathcal{C})$. The cost from the configuration $\mathcal{C}$ to the next optimal configuration $\mathcal{C}'$ is estimated by selecting the node with the lowest $f(.)$ cost. The evaluation function used is therefore defined as follows: 
\begin{equation}
    f_{\mathcal{S},T}(\mathcal{C}) = g_{\mathcal{S}}(\mathcal{C}) + h_{\mathcal{T}}(\mathcal{C})
    \label{Eq:evaluation}
\end{equation}
where $g(\mathcal{C})$ is the number of moves required to reach configuration $\mathcal{C} \subset Q_d$ from $\mathcal{S}$; $h(\mathcal{C})$ is the lower bound on the minimum number of moves required to reach the target configuration $\mathcal{T}$.

For a given target and any configuration vectors $\mathcal{T}, \mathcal{C} \in \left\{ 0,1\right\}^{d\times(2^d -l)}$, we define the following heuristic distance.

\begin{definition}
    Let $Q^d$ be a hypercube in dimension $d\in\mathbb{N}.$ Let $\mathcal{T}$ be an ordered unique subset of the vertices of $Q^d$. Let $k$ be the $k$-rule for the given puzzle. Let $\mathcal{C}$ be any given configuration of the same cardinality as $\mathcal{T}$. Then, \[ h_{\mathcal{T}}(\mathcal{C}) = \sum_{i=0}^{\#\mathcal{C}}\left\lceil \frac{|\mathcal{C}_i-\mathcal{T}_i|}{k} \right\rceil.\]
\end{definition}

The above distance represents the shortest possible number of moves by which the target is reachable if the $k$-rule blocks no configurations.\\

It is well-known that the $A*$ search is optimal if the heuristic chosen is admissible and monotonic \cite{hart1968formal}. A heuristic is \textbf{admissible} if it never overestimates the distance to the goal. Trivially, the null heuristic (set all distances to $0$) is admissible. By construction, the above heuristic is admissible. A heuristic is \textbf{monotonic} if for any configurations $\mathcal{C}, \mathcal{C}'$, $h(\mathcal{C})<d(\mathcal{C},\mathcal{C}')+h(\mathcal{C}')$. Since $d(\mathcal{C},\mathcal{C}')$ is precisely the number of moves required to reach $\mathcal{C}'$ from $\mathcal{C}$, it is clear the above heuristic is monotonic.

\subsubsection*{The successor operator ($\Gamma$)}

During the course of the A* algorithm, if $\Gamma$ is applied to a node (or configuration) $\mathcal{C}$, it expands the node $\mathcal{C}$, i.e., generates all possible configurations we can reach from $\mathcal{C}$ respect to the $k$-rule set. For a given parameter $k$, our successor operator will execute the following steps: 
\begin{enumerate}
    \item For each ring position $i$, 
    \begin{enumerate}
        \item remove $\mathcal{C}_i$ from $\mathcal{C}$ 
        \item get all the faces of dimensions $k$
        \item For each of the faces, 
        \begin{enumerate}
            \item check if the face is free (i.e. if there is only one ring on the face, the one at position $i$)
            \item save the possible move $(i,j)$ for each node $j$ of the face in a list if the face is free
        \end{enumerate}
    \end{enumerate}
    
\end{enumerate}

Given the hypercube $Q_d$, the starting and target configurations ($\mathcal{S}$ and $\mathcal{T}$),  the evaluation function $f_{\mathcal{S}, \mathcal{T}}$ and the successor operator $\Gamma$, the A* search algorithm for the sliding puzzle is given in \autoref{Algo:A*}.

\subsection{Reinforcement learning (RL)}
Similar to classical combinatorial problems such as TSP, applying the RL method to the higher-dimensional puzzle needs to be modeled as a sequential decision-making process, where the agent interacts with the environment by performing a sequence of actions to find a solution. The MDP provides a widely used mathematical framework for modeling this type of problem.

\begin{definition}
    MDP can be defined as a tuple $M = (Q^d, A, R, T, r, s)$, where 
\begin{itemize}
    \item $Q^d$ - state space $\mathcal{C}_t \in Q^d$. State space for our higher-dimensional puzzle represents any valid game configuration, i.e., a configuration accessible through a $k$-rule;
    \item $A$ - action space $a_t \in A$. Actions represent the list of possible moves with respect to the $k$-rules (e.g. performing a move of a ring from one node to another with respect to the $k$-rule);
    \item $R$ - reward function maps states and actions into real numbers. Rewards indicate how a move chosen in a particular state improves or worsens the distance to the target configuration;
    \item $T$ - transition function $T(\mathcal{C}_{t+1}| \mathcal{C}_t, a_t)$ that governs transition dynamics from one configuration to another in response to performed move;
    \item $r$ - scalar discount factor or rewarding rate, $0 <r \leq 1$. Discount factor encourages the agent to account more for short-term rewards;
    \item $s$ - horizon, which defines the length of the episode, where the episode is defined as a sequence $\left\{\mathcal{C}_t,a_t, \mathcal{C}_{t+1}, a_{t+1},\mathcal{C}_{t+2}, a_{t+2}, \cdots \right\}^s_0$. Since, in our case, the solutions are constructed incrementally, $H$ is defined naturally by the number of moves performed until the target configuration is reached.
\end{itemize}
\end{definition}

In our case, the goal of an agent acting in MDP is to find a policy function $\pi(\mathcal{C})$ that maps configuration into right move. Solving MDP means finding the optimal policy that maximises the expected cumulative discounted sum of rewards:

\begin{equation}
    \pi^* = \arg\max_{\pi} \mathbb{E} \left[ \sum_{t=0}^{s} \gamma^t R(\mathcal{C}_t, a_t)\right]
\end{equation}

Once MDP has been defined for a combinatorial optimisation problem, we must decide how the agent would search for the optimal
policy $\pi^*$. Broadly, there are two types of RL algorithms: value-based and policy-based methods. We implement in this work a value-based RL algorithm that first computes the value action function as the expected reward of a policy $\pi$ given a configuration $\mathcal{C}$ and performing a move $a$. The main difference between most value-based approaches is in how to estimate the value action
accurately and efficiently.

In the case of the sliding puzzle, the application of traditional RL faces additional obstacles. First, defining a reward policy for individual action-state pairs is impractical, as it is unclear if particular moves are helpful or unhelpful until the end of the search. The same moves may appear in winning and losing paths, and a heuristic evaluating midgame move quality proves elusive as moving between configurations very close in the Hamming metric can be arbitrarily difficult, or even impossible. Second, as the size of the state space increases exponentially with the dimension of the puzzle and the number of coloured vertices, initial random search for a path to the target configuration quickly becomes infeasible. To resolve this issue, we introduce two adjustments to the RL process for the sliding puzzle. First, we delay rewards until the end of each episode, and reward winning paths while punishing sub-optimal paths once a better path is known. Second, we introduce an initial breadth-first search starting from the target configuration to provide preliminary weights for the policy. The target is given the maximum reward to ensure it is selected from any adjacent configurations, while configurations branching from the target are given depreciating rewards based on how many steps they are from the target configuration. Below, we provide a brief summary of the process for our RL methodology on the sliding puzzle.

\begin{enumerate}
    \item Starting from the target configuration, we use breadth-first search to enumerate P configurations from the target state and assign preliminary weights to these configurations based on their distance to the target configuration.
    \item From the starting configuration, the agent makes weighted random moves according the the policy until the target configuration is reached. Configurations which have been searched in the previous step are favored over unexplored nodes.
    \item Once the target configuration has been reached, each node between the starting configuration and the target configuration is rewarded $r*(1+\lambda)^s$ where $s$ is the number of steps taken since starting and $r$ is the reward rate. Later steps are rewarded more heavily than earlier steps.
    \item From the starting configuration, the game is played again with each node chosen with probability weighted by the adjusted reward structure. In each episode, the game is terminated once more steps have been taken than the best known path achieved previously.
\end{enumerate}
Given the hypercube $Q_d$, the starting and target configurations ($\mathcal{S}$ and $\mathcal{T}$),  the maximum number of branching configuration $P$, the RL algorithm for the sliding puzzle is given in \autoref{appendix:RLAlgo}.

\subsubsection*{Initial breadth first search for RL}

In order to expedite the initial search for solutions in our RL implementation, we introduce an initial breadth first search starting from the target configuration. The idea is to widen the target around the target configuration, so that our initial random search is more likely to come upon configurations for which we know the optimal path to the target configuration. This is necessary, as our initial random search through the configuration space otherwise has to chance upon precisely the target configuration to find an initial path and assign weights for further optimisation. One can think of this initial search as allowing a golfer to aim for the green, rather than scoring a hole-in-one. An alternative would be to develop a heuristic for distance to the target configuration, or to provide a scoring mechanism to evaluate move quality during initial search, but as shown by the parity property switching even two vertices can result in unsolvable puzzles, so an effective heuristic proves elusive. 

To perform our initial search, we first determine $P\in \NN$ states to search starting from the target configuration. $P$ is determined based on the overall size of the configuration space. For most puzzles, we search through $P = 1000$ states initially. For the $d=5$ case, we search through $P = 100000$ states, as the configuration space for the $d = 5$ case is significantly larger. We expect fine-tuning $P$ would produce further benefits, but these selections were sufficient for our search to conclude with limited CPU time in the $d = 3,4$ cases and unlimited time in the $d = 5$ case. We leave the determination of an optimal selection for initial state search quantity $P$ for future work.

\subsection{Evolutionary algorithm (EA)}
In general, an evolutionary search algorithm on any fitness landscape consists of three main parts, which in the context of games: 
\begin{itemize}
	\item Initialisation: generating an initial population of game configurations, which consists of a population of starting configurations.
	\item Evaluation and selection: evaluating a given configuration in the population consists of assigning a weight to a list of moves already performed by an individual. Then, a selection force is computed by combining in a linear fashion, the number of cumulative moves and the weight and selecting a weighted (the higher the selection force, the more significant the proportion in the sample) random sample with replacement from the current population of agents to generate a new population. A detailed description of the objective function used in this work is provided in the following sections. 
	\item Mutation (or move) operation defines a set of rules or steps used to produce new configurations from the selected or initial ones. This component is elaborated further in the following subsection.

\end{itemize}
An individual in our EA is an agent with the following properties: 

\begin{enumerate}
    \item List of moves: an array that contains the different moves already performed by the agent
    \item Actual configuration: an array of $(2^d-l)-$tuple containing the occupied vertices and the corresponding ring occupying the vertex. 
    \item The hypercube $\mathcal{Q}_d$: the graph on which the puzzle is built. 
\end{enumerate}
An EA to solve this puzzle consists of evolving a population of agents that independently perform random moves on the hypercube to match all the rings to the appropriate vertex colour by minimising the number of moves. 
Two essential components must be defined to implement an EA on the higher-dimensional sliding puzzle: the fitness function or objective function and the mutation operator. 
\subsubsection*{Objective functions}
The objective function plays an essential role in any mathematical optimisation problem. To address the sliding puzzle problem computationally, the prerequisite is to define an appropriate objective function. In this section, we provide an overview of the objective function considered in the EA implementation. We aim to minimise the number of moves the agent population performs and maximise the number of matches between the ring and vertex colours. To that end, we consider two metrics: a hamming distance between the agent configuration and the target configuration and a second one that counts the number of moves an agent performs. Let $\mathcal{T}$ be the target configuration, which consists of a list of $L-$tuples (with $L=2^d -l$). The fitness of an agent $A$ is then defined as a function $f_a$: 
\begin{equation}
    f_a(\mathcal{C}, \mathcal{T}) =  \frac{1}{1+L-h(\mathcal{C},\mathcal{T})}, 
\end{equation}
where $\mathcal{C}$ is the current agent configuration and $h(\cdot, \cdot)$ the hamming distance between two configurations defined as follows: 

\begin{equation}
\label{eq:hamming}
h(\mathcal{C}, \mathcal{T}) = \sum_{i=1}^{L}{\delta(\mathcal{C}_i, \mathcal{T}_i)}
\end{equation}
and, 
$$
\delta(\mathcal{C}_i, \mathcal{T}_i) =
\begin{cases}
1 & \text{if $\mathcal{C}_i \neq  \mathcal{T}_i$ } \\
0 & \text{otherwise}.
\end{cases}
$$

\begin{algorithm}[t!]
	\tcc{$P'=\{A'_1\dots A'_N\}$: the best population\; 
		$P= \{A_1\dots A_N\}$: the initial population of $N$ agents configurations\;
		$\mathcal{D}$: a given probability distribution (L\'evy or Binomial) with parameter $p$ and $L$, where $L$ is the length of the target RNA structure\; 
		$T$: the maximum number of generations\; 
		$N$: the population size \; 
        $Q^d$: a hypercube of dimension $d$ on which the sliding puzzle is defined.\;
		$f_a(;)$: the fitness function used. It can be the hamming or any distance that allows to compare two configurations $\in Q^d$\;
		$\mathcal{T}$: the target configuration in its string representation \;
        $\mathcal{C}_b$: the configuration of the fittest agent in the population $P$\;
        $k\in \{ 1\cdots d-1\}$: the dimension of the face \;
        $\mathcal{S}$: starting configuration in its string representation\;}
	
	\KwInput{$T$,$P$, $\mathcal{D}(c), f(;), \mathcal{S}, \mathcal{T}$, $Q^d, k$}
	\KwOutput{Best population $P_b$ }
	$P_b\leftarrow P$ \tcp*{Assign the initial population to the best population }
	$t\leftarrow 0$ \tcp*{Initialise the number of generations to $0$ }
	$\mathcal{C}_b \leftarrow \mathcal{S}$ \tcp*{Set the starting configuration to be the best agent configuration.}
	\While{$t \leq T  \AND f_a(\mathcal{C}_b, \mathcal{T}) \neq 1$ } {
		$\Sigma \gets \{\mathcal{C}_i\}_{1\leq i \leq N}$ \tcp*{Get the configuration for each agent $i$ in the population $P$}
		
		$\kappa = \lfloor (N\times0.1) \rfloor$ \tcp*{The number of fit agents to copy in the next generation without mutating them.}
		$F \leftarrow \{f_a(\mathcal{C}, \mathcal{T}) | \forall \mathcal{C} \in \Sigma\} $ \tcp*{Evaluate the fitnesses of the population agent configurations to the target configuration $\mathcal{T}$ and store them in a list $F$}
  
		$E_{\kappa} \leftarrow \{A_1\dots A_{\kappa}\} \sim F$ \tcp*{copy of the $10\%$ best agent with respect to their fitnesses $F$.} 
		$S \leftarrow \{S_a(\mathcal{C},\mathcal{T}, A_i) | \forall\mathcal{C} \in \Sigma\} $ \tcp*{Evaluate the selection force of the population agent configurations}
  
		$P_S \leftarrow \{A_i\} \sim S$, where $i\in \{1,2,\dots,N-\kappa\}$\ \tcp*{Randomly sample $(N-\kappa)$ agents from $P'$ with respect to their selection forces $S$.}
		
		$M_t\leftarrow \texttt{mutate}(P_S, \mathcal{D}(c), Q^d, k)$ \tcp*{ Mutated the selected sequences using the mutation algorithm presented above.}
		
		$P_b \gets M_t \cup E_{\kappa}$\tcp*{Combine the mutated population and the best solutions to form the new population that will evolve in the next generation}
		$\mathcal{C}_b \leftarrow \arg \max_{\mathcal{C} \in \Sigma} f(\mathcal{C}, \mathcal{T})$\;
		$ t \gets t + 1$ \tcp*{Increment the time step (the number of generations)}
	}
	\caption{Main EA}\label{main_EA}
\end{algorithm}
The selection force of an agent is then defined as a linear combination of $f_a$ and the inverse of the number of moves performed so far by the agent. We write: 
\begin{equation}
    S_a(\mathcal{C}, \mathcal{T}, A_m) =  \alpha f_a + \frac{1-\alpha}{1+A_m},
\end{equation}
where $\alpha$ is a tunable parameter allowing us to balance between two objective functions: $f_a$ and the number of moves. Each agent in the population of size $N$, at time $t$ has a probability $p_i$ of being selected to the next generations $t+1$ defined as follows: 

\begin{equation}
    p_{i}(A) = \frac{S_i}{\sum_{i = 1}^{N} {S^i}},
\end{equation}
where $S_i$ is the selection force of an agent $i$ in a population of $N$-agents.
\subsubsection*{Mutation or move operator}
Each agent in the population independently makes $n$-moves at every generation of the EA. After performing random $n$ moves, the number of moves is then increased by $n$, and the agents' configuration fitness and selection force are reevaluated. 
The mutation is the process of updating the agent current configuration by randomly performing $n$-moves.

The EA explores the space $Q^d$ through its move (or mutation) operator. Given a configuration $\mathcal{C} \in Q^d$, a configuration $\mathcal{C}' \in Q^d$ is said to be an $n$-point mutation of $\mathcal{C}$ if it differs from $\mathcal{C}$ at $n$ ring positions; i.e. $h(\mathcal{C}, \mathcal{C}')=n$ where $h(.,.)$ is the hamming distance on $Q^d$. 

A mutation mode is a random variable $U$ taking values in $\{1,...,2^d - l\}$. We define $P(U=n)$ as the probability that exactly $n$ positions in a configuration $\mathcal{C}$, selected uniformly at random, undergo positional move during a mutation event. $U$ can generally be any probability distribution. We consider here the Zipf distribution:
	$$
	P(U=n)= \frac{1/n^c}{\sum_{k=1}^{L}{1/k^c}},
	$$
where $c>0$ is the value of the exponent characterising the distribution.
 Larger values of $c$ are associated with a more significant proportion of local search, while smaller values of $c$ imply a more considerable proportion of long-range search.

\subsubsection*{Constraints on the move operator}
For each position $i$ of the ring to move, some constraints need to be satisfied. We call $\mathcal{F}_i$ a $ k-$ dimensional face of the hypercube $Q^d$ and $R_i$ a ring on one of the vertices of $\mathcal{F}_i$. If the remaining vertices of $\mathcal{F}_i$ are unoccupied by any ring, then the ring $R_i$ can move or slide to occupy any of the vertices on $\mathcal{F}_i$.


We present the mutation algorithm in \autoref{appendix:mutation}. This mutation algorithm is integrated into a unified EA framework, allowing the update of agent configurations at each iteration or generation. After we apply the mutation operation to the population of agent solutions of the sliding puzzle, we evaluate the newly generated population; this is done using an objective or fitness function defined in the previous section.


\subsubsection*{Main EA}
Let \(T\) be the maximum number of generations and \(P_0\) the initial population of agents with all having as current configuration the starting configuration $\mathcal{S}$. Our EA is described in \autoref{main_EA}. The stopping criteria are two: 1) the number of generations ($t$) is equal to the max number of generations ($T$), or 2) the minimum hamming (or base-pair) distance of the best agent configuration to the target configuration is $0$ (i.e., the maximum fitness value is $1$).

\begin{table}[t!]
    \centering
       \caption{\textbf{Summary of $\text{A}^*$ search performance on the higher-dimensional sliding puzzle for $d=3, 4$ and $d=5$}. Face dimensions $k$ vary from $1$ to $d-1$ for both dimensions $d=3, 4$ and $d=5$. The success is $1$ when the algorithm terminates successfully and $0$ else. The difficulty levels range from easy (0) to hard (the most significant integer). The last column is the minimum number of moves to reach the target configuration $T$. Overall, the A* search solves the puzzle within an acceptable CPU time except the level $4$ for $d=3, k=1,2,3$, and the level $0$ for all the face dimensions $k$ for $d=5$.}
       \vspace{0.2cm}
    \begin{tabular}{cccccc}
        \toprule Dimension ($d$)&Number of uncoloured vertice ($l$)& Face dimension ($k$)& Level& Success & Min \\
        \midrule \multirow{6}{0em}{$3$}& \multirow{6}{0em}{$4$}&\multirow{3}{0em}{$1$} & $0$ & $1$& $4$\\
        \cline{4-6}
        & & & $1$& $1$ &$6$\\
        \cline{4-6}
        & & & $2$& $1$ &$10$\\
        \cline{3-6}
        & &\multirow{3}{0em}{$2$} & 0&$1$&$6$\\
        \cline{4-6}
        & & & $1$& $1$ &$7$\\
        \cline{4-6}
        & & & $2$& $1$ &$9$\\
        \midrule \multirow{12}{0em}{$4$}& \multirow{15}{1em}{$11$}&\multirow{4}{0em}{$1$} & $0$ & $1$& $4$\\
        \cline{4-6}
        & & & $1$& $1$ &$8$\\
        \cline{4-6}
        & & & $2$& $1$ &$6$\\
        \cline{4-6}
        & & & $3$& $1$ &$8$\\
        \cline{4-6}
        & & &4&0&-\\
        \cline{3-6}
        & &\multirow{4}{0em}{$2$} & 0&$1$&$6$\\
        \cline{4-6}
        & & & $1$& $1$ &$7$\\
       \cline{4-6}
        & & & $2$& $1$ &$6$\\
        \cline{4-6}
        & & & $3$& $1$ &$7$\\
        \cline{4-6}
        & & &4&0&-\\
         \cline{3-6}
         & &\multirow{4}{0em}{$3$} & 0&$1$&$8$\\
        \cline{4-6}
       & & & $1$& $1$ &$9$\\
       \cline{4-6}
       & & & $2$& $1$ &$10$\\
        \cline{4-6}
       & & & $3$& $1$ &$10$\\
       \cline{4-6}
        & & &4&0&-\\
        \midrule \multirow{5}{0em}{$5$}& \multirow{5}{1em}{$26$}&1& $0$ & $0$& -\\
        \cline{3-6}
        & & 2& 0& 0 &-\\
        \cline{3-6}
        & & 3& $0$& $0$ &-\\
        \cline{3-6}
        & & 4& 0& 0 & -\\
        \bottomrule
    \end{tabular}
    \label{tab:Astart_performance}
\end{table}

\section{Experimental Results}
Using the above methods, we computationally study the sliding game for dimensions $d=3, d=4$, and $d=5$ for face dimensions $1\leq k<d$ across different difficulty levels. This section provides a detailed analysis of the performances of each algorithm. We first surf through each method's individual performances, then present a comparative and CPU analysis. 
\begin{figure}[t!]
    \centering
    \includegraphics[width=1.0\linewidth]{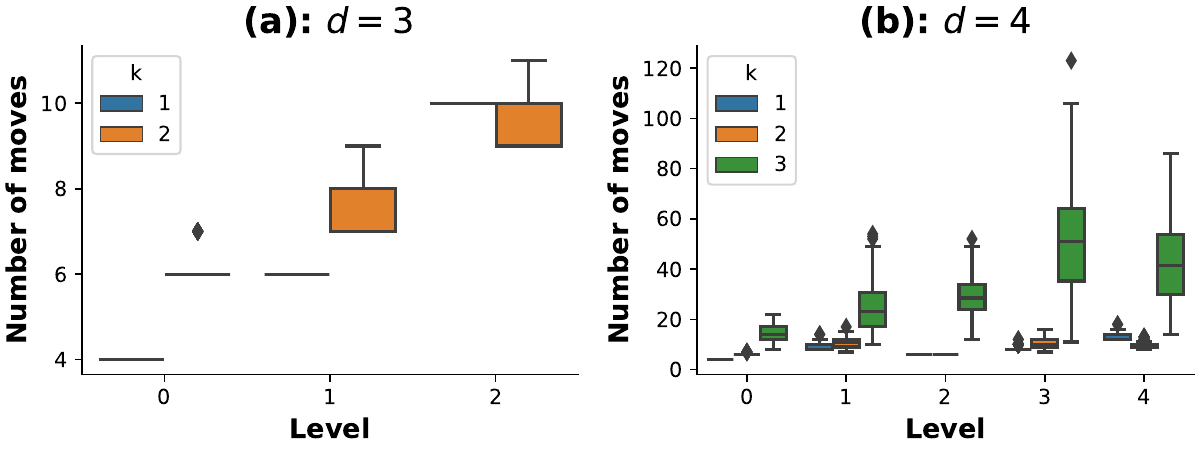}
    \caption{\textbf{RL Performances: distributions of the number of moves across different difficulty levels for dimensions $d=3,4$ with different face dimensions $k=\left\{1, \cdots, d-1\right\}$}. The number of moves increases with the face dimension and across difficulty levels, with a significant gap between level 4 and the other levels.}
    \label{fig:RL}
\end{figure}

\subsection{Individual Performances}
Depending on the initial position of the rings, the difficulty of solving the cubic puzzle can move from easy to very hard or impossible. For the puzzle of dimension $3$, three difficulty levels out of the four reported in \cite{beyer2023higher} are analysed in this work. One level has been reported to be impossible because the initial and target configurations are part of distinct disconnected components in the configuration space. The performances of five levels of difficulty for the puzzle of dimension $4$ and one level for dimension $5$ are also analysed in this section. We set each algorithm's maximum running CPU time to $3600s$, except for the puzzle of dimension $d=5$.
\subsubsection*{A* Search Performance} 
Applying the A* search on the higher-dimensional sliding puzzle is straightforward, and compared to stochastic methods, no tuning parameter is necessary. \autoref{tab:Astart_performance} shows the minimum number of moves for three puzzles with $d = 3$, five puzzles with $d=4$, and one puzzle with $d=5$ under applicable $k$ rules and across different levels of difficulty. We observe that the A* search produces optimal puzzle solutions for dimension $d=3$ and $d=4$ except for the level $4$. For different dimensions, face dimensions and difficulty levels, the minimum number of moves varies between $4$ and $10$. For the dimension $d=3$, the number of moves increases monotonously with the difficulty levels for face dimensions $k=1$ and $k=2$.

In contrast to the dimension $d=3$ for $d=4$, the number of moves increases with the difficulty level only for $k=3$. The latter may be because of the fact the difficulty levels were generated for $k=d-1$ only. For $d=4$, $k=2$, the number of moves is $7$ for both difficulty levels 1 and 3, where the number of moves is 6 for levels 0 and 2. The same applies for $k=1$, except that the number of moves is $4$ and $6$ for levels 0 and 2 and $8$ for levels 1 and 3. 

Compared to RL and EA methods, A* search outperforms success in finding the optimal solution for dimensions 3 and 4 across different difficulty levels, except for level 4 with $d=4$. The A* search did not find any optimum solution within the maximum running time allowed for level 4 of the puzzle of dimension $d=4$ and level 0 of dimension $5$ under different $k$-rules. 

\subsubsection*{RL Performance}

To run our RL algorithm, we first determine an appropriate number of states branching from the target state to enumerate before beginning initial random search. For all experiments, 100 branching states were searched, except for the $d=5$ puzzle where $100000$ states were searched. We set a learning rate of $\alpha = .05$ and a discount rate of $\gamma = .95$. For each experiment, 1000 iterations of the search were performed over 150 experiments. 

\autoref{tab:sum_performance_RL} shows the minimum, maximum, and median number of moves for three puzzles with $d=3$ and four puzzles with $d=4$ under applicable $k$ rules. We observe that for all puzzles aside from $d=4$ and $k=3$, the RL search produces optimal solutions within the experimental population. Compared to $A^*$ search and our EA, our RL search method appears capable of producing significantly higher maximum move solutions for more difficult puzzles, suggesting more significant variance in the algorithm's performance. \autoref{fig:RL} shows the distributions of the number of moves for the dimensions $d=3$ and $d=4$ under applicable $k$-rules and across different difficulty levels. Nevertheless, solutions are produced for all difficulty levels of the problem considered here for the puzzle of dimensions $3, 4$ and $5$.

\subsubsection*{EA Performance}
One of the challenging tasks when engineering an EA is to tune the hyper-parameters, such as the mutation rate, the population size and the selection function. To run our EA efficiently, we first analyse its critical parameters, such as the mutation rate $c$ and the selection force ($\alpha$) on the puzzle of dimension $3$, level $0$. Then, using the best algorithm parameter configuration, we analyse the agents' performance to solve different levels of the puzzle of dimensions $3, 4$ and $5$ while minimising the number of moves. 

\paragraph{Parameters analysis}

The best mutation parameter $c^*$ is the one that has the lowest median number of generations and the highest score in hitting the exact minimum number of moves.  Details about the mutation parameter analysis are presented in \autoref{appendix:EAparameter}. 


It is essential to notice on \autoref{fig:mu_success_ratio}a that when looking at the 6 moves' frequency,  $c^* = 1.8$ is the optimal mutation rate. For the benchmarks in this work, we set $c^*$ to $1.8$, simply because for $c=1.8$, $23$ agents can solve the puzzle with exactly $6$ moves while keeping the same success rate. Using the optimum value $c^* =1.8$ allows simultaneous search at all scales over the landscape. New mutations often produce nearby configurations (one move) but occasionally generate configurations far away in configuration space (several moves), consequently increasing the number of moves needed to solve the puzzle. 

We set the impact of the selection force parameter $\alpha$ to $0.2$, i.e., we minimise $80\%$ of the time the number of moves and only $20\%$ of the time the distance to the target. The EA also incorporates elitism to keep track of agents that are the closest to the target configuration. We will use a Levy mutation, i.e. $c^* = 1.8$, and the population size and the maximum number of generations are both set to $1000$.

\paragraph{Performance analysis}
Considering the previously best parameter of our EA for the puzzle of dimension 3 with level 0, we ran our EA on the two other levels of the same dimension and also on the puzzle of dimensions 4 and 5 across different levels and under the $k$-rule move.  
\begin{figure}[t!]
   \centering
   \includegraphics[width=1.0 \linewidth]{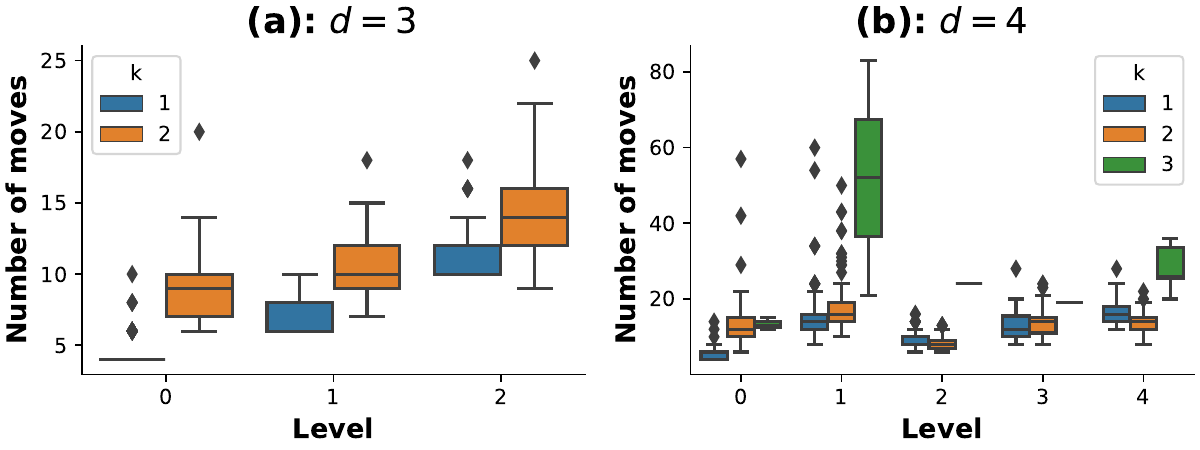}
   \caption{\textbf{EA Performances: distribution of the number of moves for different levels of difficulties and with puzzle parameters $k=2$, $d=3$ and $l=4$}. The minimum, maximum and median number of moves increases with difficulty. The minimum number of moves for different difficulty levels is $6, 7$, and $10$, with a frequency of $14, 2$, and $3$ out of 150 runs. }
   \label{fig:difficulty_levels}
\end{figure}
\autoref{tab:sum_performance} shows descriptive overall performance statistics for different puzzle and face dimensions across various levels. As the face dimensions increase, the EA performance degrades for dimensions 3 and 4, i.e., the median number of moves to solve the puzzle increases. The lower the face dimension is, the fewer constraints in the puzzle, allowing random moves to be more often valid and giving more chances of exploring the configuration space faster. Similarly, for $k=1$, the median number of moves needed to solve the puzzle of dimension $3$ increases with the difficulty level, except that the median number of moves for face dimension $k=1$ stays lower than that of face dimension $k=2$. 

In contrast to the puzzle of dimension $3$, the EA performance for $d=4$ sometimes degrades with the difficulty level, except from level 2 to level 2 and from level 2 to level 3. For different values of dimension face, levels 1 and 4 appear to be the most difficult, whereas the easier one is level 2. For $k=3$, the EA performance was considerably degraded, drastically dropping from 100\% to 2\% success rate for level 0, $1.3\%$ for level 1,  $0.06\%$ for levels 2 and 3, and $4.6\%$ for level 4. 

\begin{figure}[t!]
    \centering
    \includegraphics[width=1.0 \linewidth]{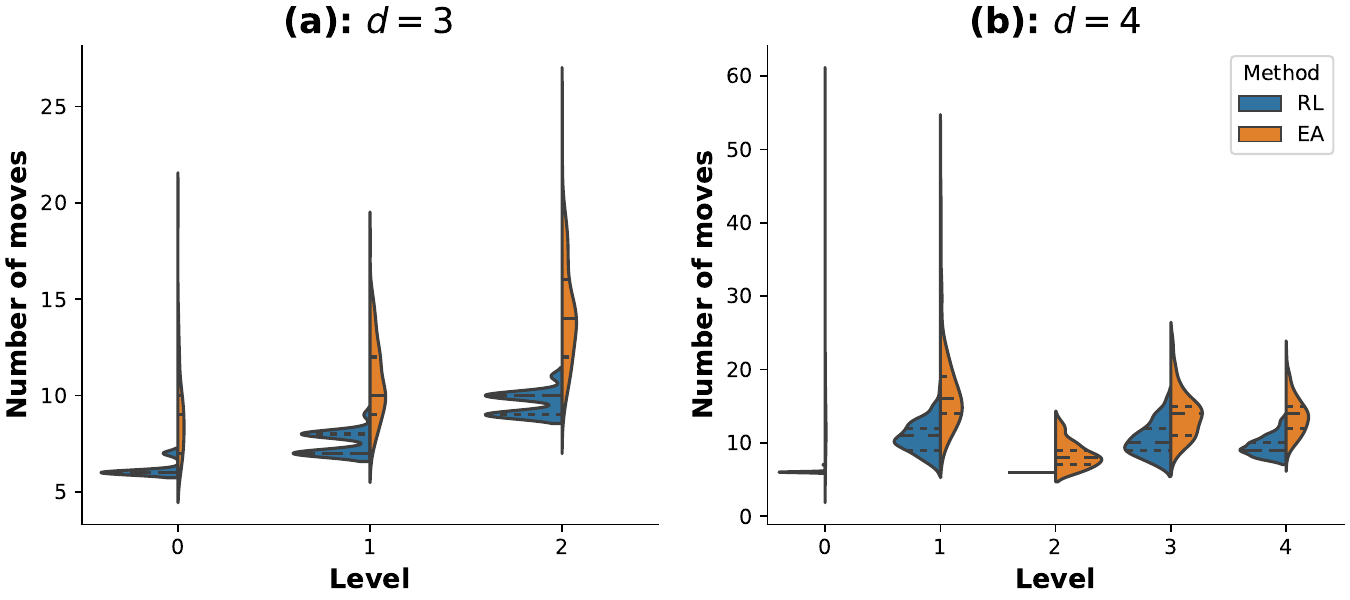}
    \caption{\textbf{EA \emph{vs.} RL number of moves distributions across different difficulty levels for the face dimension $k=2$}. \textbf{(a)} For the puzzles of dimension $d=3$, the EA number of move distributions have a wider spread of values, while RL distributions exhibit a more concentrated distribution around the median. \textbf{(b)} For the puzzle of dimension $d=4$, RL's number of moves shows a more peaked distribution, whereas EA's number of moves exhibits a more evenly spread distribution across its range.}
    \label{fig:RL_versus_EA}
\end{figure}
To illustrate the EA performance for dimensions 3 and 4 puzzles in more detail, \autoref{fig:difficulty_levels} shows the distribution of the number of moves for each difficulty level and under different $k$-rules. The higher the levels are, the higher the minimum number of moves to solve the puzzle. For dimension $d=3$, we observe a clear drop in performance as the puzzle difficulty increases, e.g., for level 0, $23$ agents can solve the puzzle with $6$ moves. In contrast, no agents solve the puzzle in higher levels with $6$ moves. Additionally, fewer agents in level 1 than in level 0 solve the puzzle in $7$ moves, where none of the agents can solve the puzzle in 7 moves for level 2. This observation holds for both face dimensions $k=1$ and $k=2$, and a similar observation can be made for the puzzle of dimension 4, except that the success rate for face dimension $k=4$ is very close to zero, which results in very few points in the displayed distributions. 

Overall, the EA's performance degrades as the difficulty levels and the face and puzzle dimensions increase. Compared to the A* search, the EA successfully solves the puzzle of dimension 5 under different $ k$-rules except for $k=4$; this demonstrates the power of stochastic searching techniques in a vast combinatoric space of possible solutions. Widespread, an EA with a Levy mutation can sample short paths for a higher-dimensional sliding puzzle for $d=4, 5$ and different values of $k$, which allows us to go beyond theoretical results obtained in \cite{beyer2023higher}.

\subsection{EA \emph{vs.} RL Performances}
Among the three algorithms benchmarked in the previous section, EA and RL techniques areå both choose the next move with respect to some probability proportional to the total reward/fitness of the path after adding that configuration to the solution path. This leads to random solution outputs. It is, therefore, essential, when comparing the EA and RL performances, to analyse the distribution of the minimum number of moves produced by each algorithm.

When only looking at the success rate, from Tables \ref{tab:sum_performance} and \ref{tab:sum_performance_RL} we notice that the RL method successfully solves the puzzles for dimension $d=3, 4$ and $d=5$ under different $k$-rules and across different difficulty levels. In contrast, the EA method failed to solve the puzzles of dimension $4$ for the face dimension $k=3$ with an overall success rate of less than $5\%$ and one puzzle of dimension $5$.

Furthermore, we assessed the distributions of the number of moves produced by both methods for the puzzle of dimensions $3$ and $4$ for face dimension $k=2$. \autoref{fig:RL_versus_EA} shows the distributions of the number of moves for both EA and RL methods. We noticed a less significant gap between the median and the variance of the distribution of the number of moves produced by our EA method and that of the RL method ($p$-value $=0.12$). For the puzzles of dimension $d=3$ and face dimension $k=2$, the EA distributions have medians $9, 10$, and $14$, whereas RL distributions are $6, 7$ and $10$, which results in an average gap of $3$ moves. Besides the median number of moves, the standard deviation of RL distributions is also lower than that of EA. The difference in the median number of moves is more significant, with a $p$-value of $0.03$ when comparing the distribution data for the puzzle of dimension $d=4$ and face dimension $k=2$. 
\begin{figure}[t!]
    \centering
    \includegraphics[width=1.0 \linewidth]{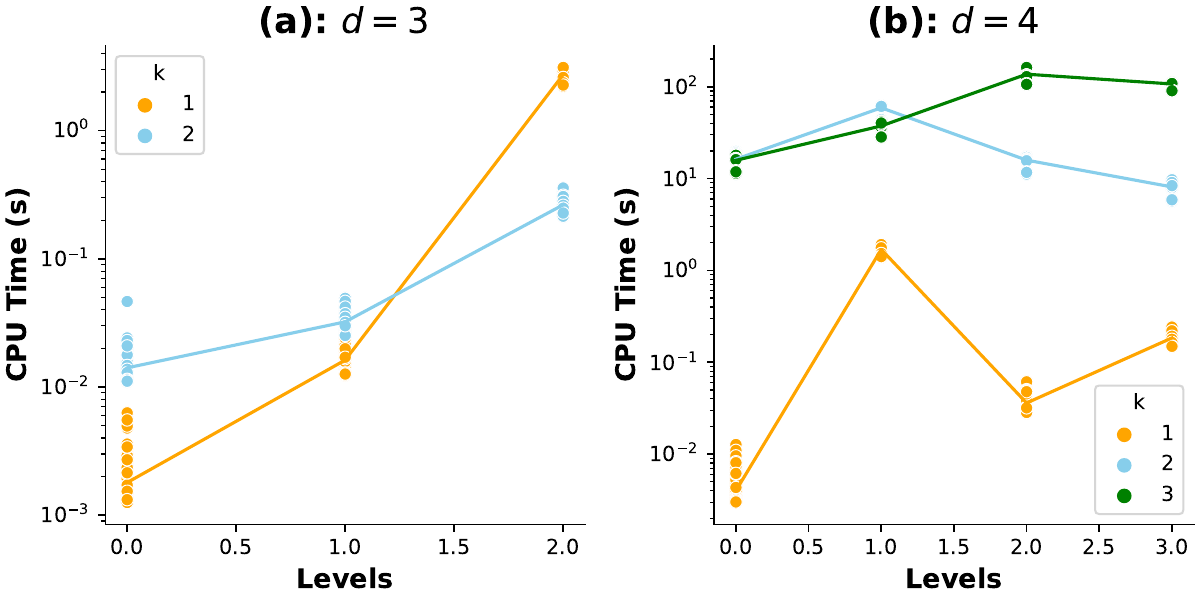}
    \caption{\textbf{A* search CPU time plotted against the difficulty levels for different face dimensions}. The dot points are different measurements, and the solid lines connect the medians. The A* algorithm consistently requires less computation time for smaller values of $k$ across all levels and puzzle dimensions, demonstrating higher efficiency than high values of $k$.}
    \label{fig:AS_CPU_Time}
\end{figure}
When comparing both results to the solved puzzle by A* search, since the A* outputs the exact solution when it stops, we notice very similar performance in terms of the minimum number of moves found. For the puzzles of dimension $3$, both methods, RL and EA, successfully solved the puzzle at least once with the minimum number of moves. The capacity of EA and RL to minimise the number of moves diminishes with the dimension increase and the difficulty levels, except for the face dimension $k=1$. In most cases, when both methods fail to produce the minimum number of moves, their approximated minimum number of moves is equal except for the puzzles of face dimension $k=3$, where EA fails to solve the puzzle. The principal advantage of the RL method over EA is when EA fails completely to produce an output, similar to the A* search. 

\subsection{CPU Time Analysis}
Another aspect of our computational study is the time complexity of different algorithms evaluated in our work. This section provides a quick analysis of the A* search algorithm and a brief empirical comparison of  the its running CPU time compared to the RL and EA algorithms.
\subsubsection*{A* Time Complexity}

The computational complexity of A* search is known to derive from the branching factor of the underlying state space. That is, if the average branching factor of each searched state is $b\in\mathbb{N}$ and the optimal solution is of length $n$, the complexity of A* search will be $O(b^n)$\cite{Russell2003artificial}. Hence, we show the following:


\begin{proposition}
The complexity of A* search on the cubical sliding puzzle with $k$-rule and optimal solution length $n$ is at least $O((2^k-1)^n)$
\end{proposition}
\begin{proof}
From \cite{beyer2023higher}, it is known that with chosen $k$-rule for $k$, if a configuration can be branched from there are at least $2^k-1$ moves which can be made with any mobile vertex. Thus, the branching factor of each state is at least $2^{k}-1$. The conclusion follows from the complexity of A* search. 
\end{proof}

For example, we may conclude that the complexity of A* search for a face dimension $k=4$ is at most $O(7^n)$ for optimal solution $n$. 
\begin{figure}[t!]
    \centering
    \includegraphics[width=1.0 \linewidth]{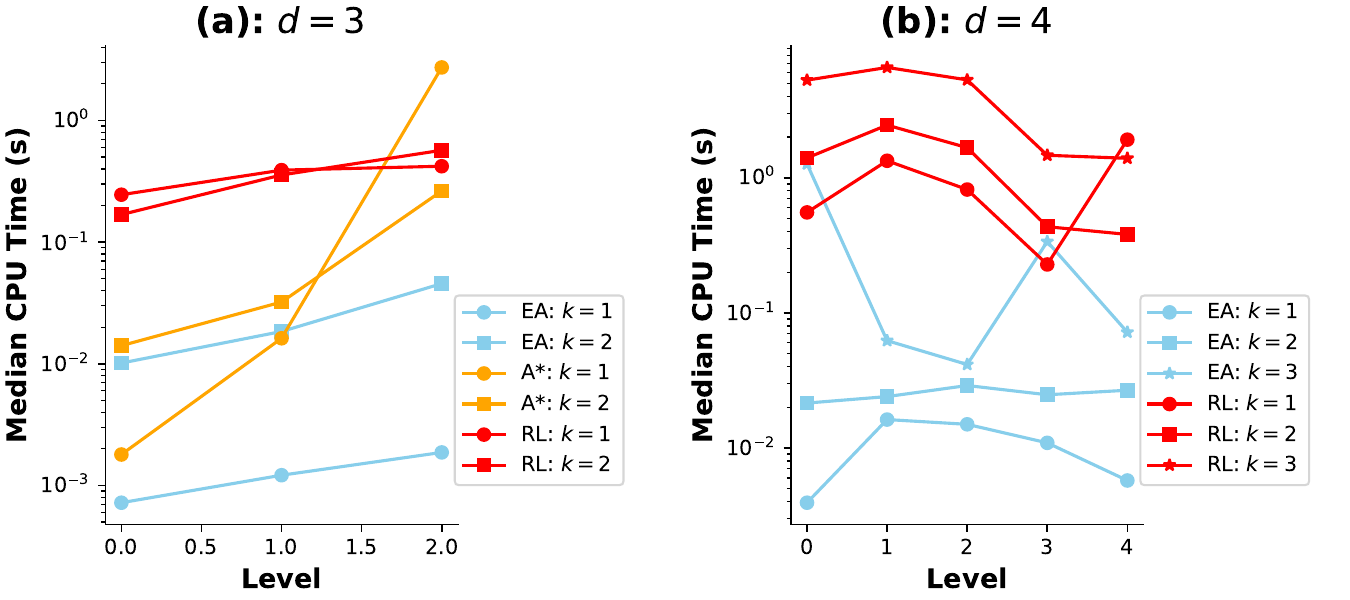}
    \caption{\textbf{A* search \emph{vs.} Comparison of CPU times for A*, RL, and EA methods across various puzzle difficulty levels for all face dimensions}. The results indicate that EA outperforms A* and RL methods, especially for higher difficulty levels, with reduced processing time. \textbf{(a)} For the puzzles of dimension $d=3$, RL exhibits significantly higher CPU usage, particularly when compared to EA, except for level $3$. \textbf{(b)} EA algorithm consistently requires less CPU time across different difficulty levels and face dimensions, highlighting its computational efficiency compared to RL.}
    \label{fig:RLEA_CPU_Time}
\end{figure}
This is a worst-case complexity assumption since the difficulty level influences the search algorithm greatly, as our previous analysis pointed out. We further present in \autoref{fig:AS_CPU_Time} the empirical CPU time measurements for all the A* search solved puzzles we previously presented. We notice the CPU time depends not only on the face dimension (or $k$-rules) but also significantly on the puzzle difficulty levels. The CPU time increases with difficulty levels across puzzle dimensions $d=3$ and $d=4$ while staying generally consistent with the face dimension. 

\subsubsection*{RL and EA CPU Time Analysis}

Comparing CPU time between exact and heuristic algorithms is not straightforward since the latter does not ensure optimality. On the one hand, our RL implementation performance strongly depends on the number of iterations and the number of branching states, which also implies a strong dependency between the CPU time and these two parameters. On the other hand, our EA performance depends on the population size and the number of generations we led it to run; these two parameters also influence the CPU time. For RL and EA, the number of branching of population size represents exactly the number of explored paths through the algorithm, and the CPU times are measured for a fixed number of branching and population size. To compare the RL and EA CPU time measurements to that of the A* search, we first retained only one of the puzzles successfully solved by the A* search and then divided the CPU time measurements obtained by both methods by their respective number of explored paths. 

\autoref{fig:RLEA_CPU_Time}a shows the CPU time measurements for all the A* search solved puzzles of dimension $d=3$ compared to RL and EA CPU time. Similarly to the A* search method, the computational time strongly depends on the puzzle dimension, face dimension, and difficulty levels, except that the gap between the face dimensions $k=1$ and $2$ is less significant for the RL method. The EA method demonstrates less computational times for the puzzles of dimension $d=3$ across different face dimensions and difficulty levels despite its poor performance in terms of minimising the number of moves. Furthermore, we compare the CPU time of RL and EA methods for the puzzles of dimension $d=4$. \autoref{fig:RLEA_CPU_Time}b shows the CPU time measurements plotted against different difficulty levels for all values of $k$. In contrast to the puzzle of dimension $d=3$, the CPU time decreases with the difficulty levels but stays consistent with different face dimensions.
\section{Conclusion}

In this work, we studied the performance of three computational methods, EA, RL, and A* search, on the higher-dimensional puzzle for dimensions $d=3$, $4$, and $5$. The first part of our work was designing the algorithms and methods for the higher-dimensional sliding puzzle. The second part was benchmarking different approaches and providing a plausible analysis of each method's performance and CPU time.

For the sliding puzzle of dimension $d=3$, all tools have proven excellent performance across different difficulty levels in minimising the number of moves except our EA method, especially for the face dimension $k=2$. The distributions of the number of moves produced by our EA had higher variance, possibly because of the choice of the implemented mutation scheme. When comparing the CPU time for the overall solved puzzles, A* search and EA were faster in producing solutions than RL despite our EA being less efficient in minimising the number of moves. 

For the puzzle of dimensions $d=4$, on the one hand, the A* search method successfully finds the shortest path, i.e., the smallest consecutive number of configurations that lead to the target configuration, whereas the RL and EA methods failed most, or when they do, the standard deviation of the distribution of the number of moves produced is high. On the other hand, A* search failed to solve the puzzle of dimension $d=4$ for level $4$ and all the puzzles of dimension $d=5$, whereas RL and EA successfully solved both of the puzzles. Similar to the dimension $d=3$, our EA ran in less time but failed in most cases to minimise the number of moves except for the puzzle of dimension $d=5$, for which A* search was unable to produce any solution.

Although the A* search performed well on the puzzle of dimensions $d=3$ and $d=4$, its time complexity remains the main challenge, as for many graph search techniques. Due to this, it becomes almost impossible to run it on puzzles of dimensions higher than four or even harder ones of dimension $4$ such as level 4. While EA and RL differ in their underlying principles and applications, they have proven similar potential to solve higher-dimensional sliding puzzles. Their complementary strengths have led to the development of hybrid approaches that leverage the global search capabilities of EAs and the sequential decision-making strengths of RL \cite{song2024reinforcement}. Combining both methods will directly improve intuition and elevate the challenges encountered. 

Another future work will be to build a full web application for the high-dimensional puzzle by integrating a backend and a better visualisation for the puzzles of dimensions higher than three. The latter will make the puzzle broadly available to human players and collect the puzzle data to enable further improvement to our algorithm or develop more refined techniques that exploit these data, e.g., supervised learning. It will also be interesting to compare our experimental results to human performances. 
\section*{Data availability}
This manuscript's codes and benchmark data are available at
\url{https://github.com/lemerleau/CubicPuzzle.git}. We also provide the scripts used for the figures and CPU Times analyses.
An online puzzle version is available on \url{https://service.scadsai.uni-leipzig.de/sliding-puzzle/}.
\section*{Funding}
We would like to acknowledge partial funding from the Alexander von Humboldt Foundation, the BMBF and the Saxony State Ministry for Science.

\section*{Acknowledgments}
We thank Peter Voran for the helpful discussions and Martin Miller for contributing to developing the new frontend and backend software to support this paper.

\bibliographystyle{unsrt}
\bibliography{main}

\newpage
\appendix 

\section{EA parameters tuning}\label{appendix:EAparameter}
To support the choice of the EA parameters used in this work, we analyse the impact of the mutation rate $c$ for a fixed selection force parameter $\alpha=0.2$.
There are three dimensions ($d=3$, $4$, and $d=5$) of the puzzle we studied, and for each of these dimensions, there are different levels of difficulty. For the mutation parameter analysis, we used the sliding puzzle of dimension $d=3$, with $l=4$ labelled vertices and $k=2$ for the dimension of the faces. We set the maximum number of generations $t$ to $1000$ and the population size $N$ to $1000$. The stopping criteria are two: 1) the number of generations ($t$) is equal to the max number of generations ($t$) or 2) the hamming distance of the best agent configuration solution to the target configuration is $1$.

\begin{figure}[H]
    \centering
    \includegraphics[width=1.0 \linewidth]{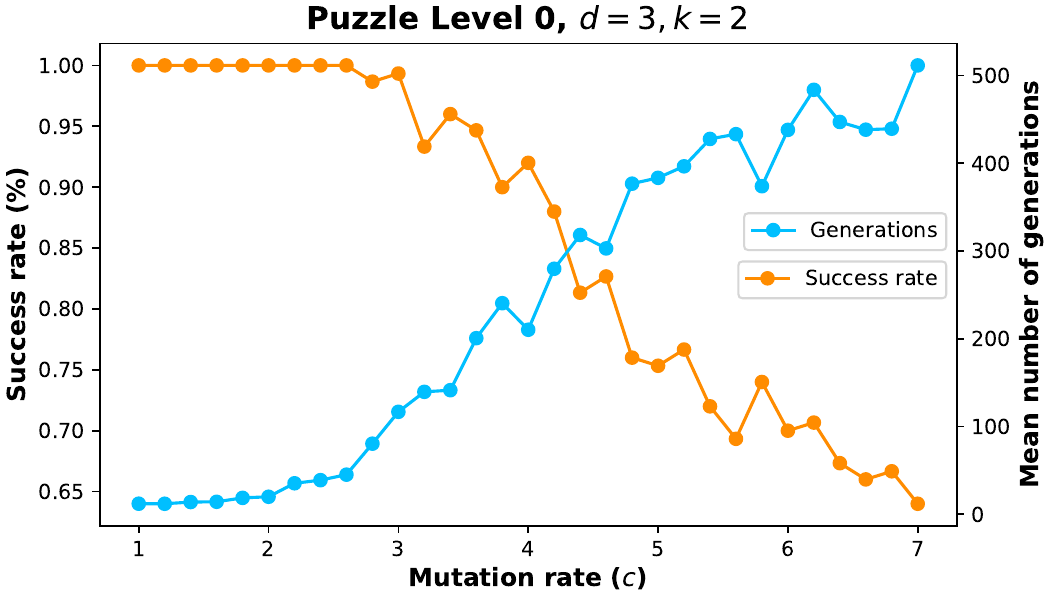}
    \caption{\textbf{EA: The success rate and the mean number of generations needed to reach the target configuration $T$ with respect to the mutation parameter $c$ (respectively in blue and orange)}. Performance degrades with the increase of $c$, and the number of generations increases with $c$.}
    \label{fig:mu_success_rate}
\end{figure}

\begin{figure}[H]
    \centering
    \includegraphics[width=1.0 \linewidth]{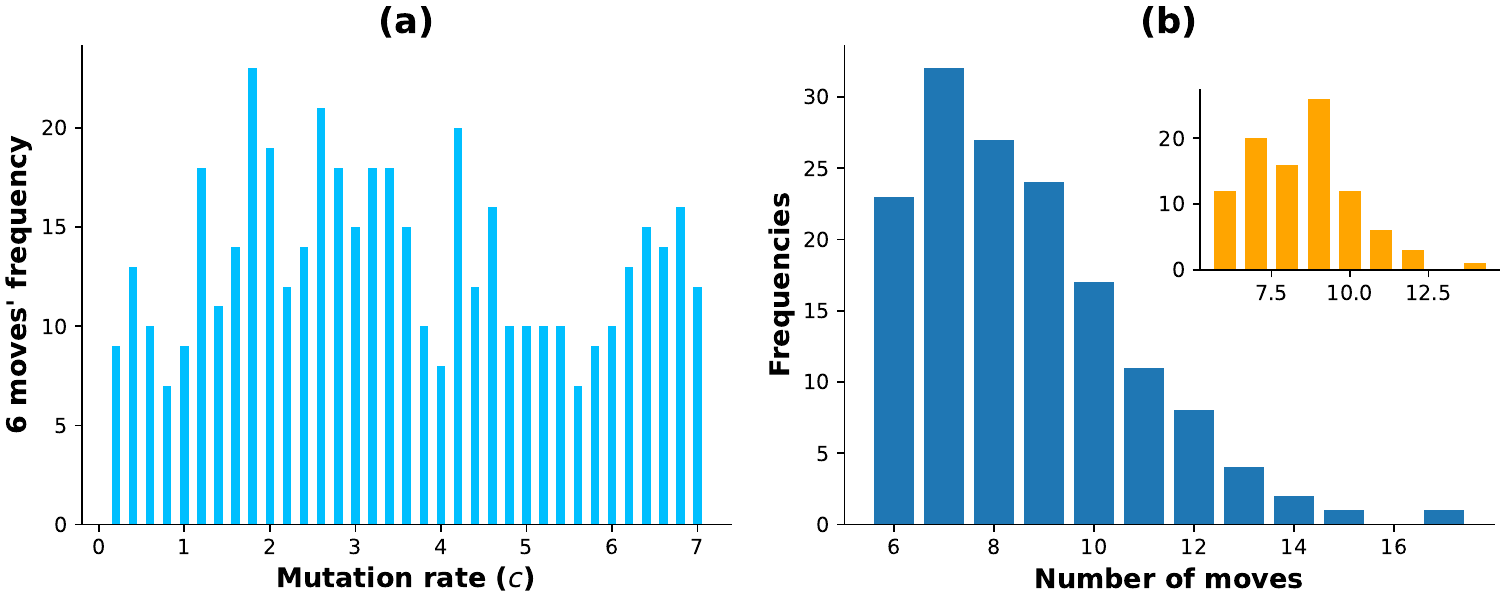}
    \caption{\textbf{EA: 6 moves' frequency \emph{vs.} mutation parameter $c$ and the frequency of the number of moves for the best parameter $c^*$ for the higher-dimension puzzle with parameter $d=3, k=2$ and $l=4$}. (a) The 6 move frequency with respect to the mutation parameter $c$. Based on the 6 moves' frequency, the best mutation parameter is $c^*=1.6$. (b) The frequency of the number of moves. For the mutation parameter $c^*=1.8$,  the EA reaches the target configuration $23$ times out of 150 trials with exactly six moves. On average, it performs ten moves to reach the target configuration $T$, whereas for $c=7$ (the inset figure), it reaches the target 12 times with six moves and the success rate for $c=1.4$ (100\%) almost doubled that of $c=7$ (60\%).}
    \label{fig:mu_success_ratio}
\end{figure}


\section{Different puzzles analysed in this work}
\label{appendix:data}
\begin{table}[H]
    
    \caption{For each of the $(d,k,l)$-puzzles, the initial coloured vertices are fixed. For example, in the $(3, 2, 4)$-puzzle, the fixed vertex colours are $[(0,w), (1,p), (2,w), (3,w), (4,g), (5, r), (6, b), (7,w)]$. The labels $w,p,r,g$, and $b$ are white (i.e. unlabelled or uncoloured), purple, red, green and blue and the number of the vertex's positions. In such a setting, a move, when possible, will consist of changing the vertex position of a given ring. For example, performing a single move from the initial configuration $[(4,r), (1, p), (5, b), (6, g)]$ can lead to the following configuration $[(4,r), (1, p), (5, b), (2, g)]$.
    }
    \centering
    \vspace{1cm}
    \begin{tabular}{p{1.5cm}p{1.5cm}p{2cm}p{4cm}p{5cm}}
         \toprule \textbf{Cube dimension} ($d$)& \textbf{Face dimension} ($k$)& \# of \textbf{unlabelled vertices} ($l$)& \textbf{Starting configuration} & \textbf{Target configuration}  \\
         \midrule
         $d=3$& $k=2$ &$l=4$ & [(4,r), (1, p), (5, b), (6, g)] & [(4, g), (1, p), (5, r), (6, b)] \\
         \hline
         $d=3$& $k=2$ &$l=4$ & [(4,p), (1, b), (5, r), (6, g)] & [(4, g), (1, p), (5, r), (6, b)]\\ 
         \hline
         $d=3$& $k=2$ &$l=4$ & [(4,r), (1, b), (5, g), (6, p)] & [(4, g), (1, p), (5, r), (6, b)] \\ 
         \hline
         $d=3$& $k=2$ &$l=4$ & [(4,r), (1, p), (5, g), (6, b)] & [(4, g), (1, p), (5, r), (6, b)]\\
         \hline
         $d=4$& $k=3$ &$l=11$ & [(13,r), (0, p), (11, g), (3, b), (7,y)] & [(0, g), (3, p), (13, r), (11, b), (7,y)]\\
         \hline
         $d=4$& $k=3$ &$l=11$ & [(15,r), (2, p), (1, g), (9, b), (8,y)] & [(1, g), (9, p), (8, r), (2, b), (15,y)]\\
         \hline
         $d=4$& $k=3$ &$l=11$ & [(13,r), (0, p), (4, g), (8, b), (11,y)] & [(8, g), (0, p), (13, r), (11, b), (4,y)]\\
         \hline
         $d=5$& $k=4$ &$l=26$ & [(30, g), (29, p), (0, r), (17, b), (4, y), (1, o)] & [(17, g),
         (30, p),
         (0, r),
         (29, b),
         (4, y),
         (1, o)]\\
         \bottomrule
    \end{tabular}
    \label{tab:my_label}
\end{table}

\section{EA: Mutation algorithm} \label{appendix:mutation}

\begin{algorithm}[H]
	\tcc{$P'=\{A'_1\ldots A'_N\}$: the mutated population of agents\; 
		$P= \{A_1\ldots A_N\}$: a list of $N$ agents to mutate\;
		$Q^d$: a hypercube of dimension $d$ on which the sliding puzzle is defined.\;
		$\mathcal{D}$: a given probability distribution (Zipf) with parameter $c$.\;
        $k\in \{ 1\cdots d-1\}$: the dimension of the face $\mathcal{F}$ \;
        $\mathcal{F}_r$: the face associated with the ring $r$\;
        $N$: the population size \;} 
	\KwInput{$P$, $\mathcal{D}(c)$, $Q^d, k$}
	\KwOutput{$P'$} 
	$ \{B_i\} \sim \mathcal{D}(c)$, where $i\in \{1,2,\dots,N\}$ \tcp*{Draw $N$ random numbers that follows a given distribution $\mathcal{D}$ (Zipf in this case). $B_i$ is the number of positions in a configuration to slide}
	
	\For{ $i \in \{1, 2, \dots, N\}$ }{
		$A'\leftarrow P_i$ \tcp*{Make a copy of the agent $A_i \in P$ to $A'$ }
        $\mathcal{C}_a \leftarrow$ get the current configuration of the agent $A'$\; 
		\For{ $j \in \{1,2,\dots, B_i\}$}{
			$r \in \{1,2,\dots, 2^d-l\} \sim \mathcal{U}$ \tcp*{select uniformly a random position in the configuration $\mathcal{C}_a$}
			  $\mathcal{F}_r \leftarrow $ get the $k$-dimensional face of $Q^d$ where the ring $R_r$ occupies the vertex at position $r$

         \If{$R_r$ is a k-free state with respect to $\mathcal{F}_r$}{
                $m \in \mathcal{F}_{r} \sim \mathcal{U}$ \tcp{select randomly a vertex on the face $\mathcal{F}_r$}
                move the ring $R_r$ to the vertex $m$\;
                $Add$ the move $(R_r, R_m)$ to the list of moves of agent $A'$\;
                Update the current configuration $\mathcal{C}_a$\;
            }
     }
          Update the current configuration of the agent $A'$\;
          $P' \gets P'\cup A'$ \tcp*{Add the mutated agent $A'$ to the list $P'$} 
		}
		
	\caption{Agent mutation algorithm}\label{Algo:move}
\end{algorithm}

\section{RL algorithm}\label{appendix:RLAlgo}
\begin{algorithm}[H]
	\tcc{$Q^d$: a hypercube of dimension $d$ on which the sliding puzzle is defined\;
        $\alpha$:, the learning rate.\;
        $\gamma$:, the discount rate.\;
        $k\in \{ 1\cdots d-1\}$: the $k$ rule for the given puzzle.\;
		$\mathcal{S}\subset Q^d$: the initial ordered unique configuration of rings on the hypercube $Q^d$.\;
        $\mathcal{T}\subset Q^d$: the target ordered unique configuration of rings on the hypercube $Q^d$.\;
        $P\in \mathbb{N}$:, the number of branching configurations from the target configuration to enumerate.\;
        $i$:, the initial weighting for each configuration.\;
        $w(\mathcal{C})$:, a function returning the current weight for each configuration $\mathcal{C}$. Initialised at $i$.\;
        $s(\mathcal{C})$:, a function returning a successor configuration to the current one based on the current weight function $w(\mathcal{C}).$\;
        $I$:, the number of iterations to run.\;
        $M$:, the best path found.\;
        }
	\KwInput{$\mathcal{S}$, $\mathcal{T}$, $Q^d$, $P$}
	\KwOutput{$M$}
        $O \leftarrow \{\mathcal{T}\}$\ \tcp*{Edge configurations on an expansion of the target configuration.}
        $w(\mathcal{T}) \leftarrow 1$ \tcp*{the target configuration is assigned maximal weight.}
        $c \leftarrow 1$ \tcp*{A counter value.}
        \While{$n(\mathcal{T})<N$}{
            $O^* \leftarrow\emptyset$\;
            \For{$\mathcal{C}\in O$}
            {
                $O^* \leftarrow O^*\cup n(\mathcal{C})$ \tcp*{Attaching neighbors of $\mathcal{C}$ to $O^*$.}
                \For{$\mathcal{C}^*\in n(\mathcal{C})$}
                {
                    $w(\mathcal{C}^*)\leftarrow\frac{1}{2^c}$ \tcp*{Assigning decreasing weight to each value extending from the target configuration.}
                }
            }
            $c \leftarrow c+1$\;
            $O \leftarrow O^*$\;
        }
        $o \leftarrow\infty$ \tcp*{Initialise optimal configuration to infinity.}
        \For{$k\in I$}{
            $c \leftarrow 0$\;
            $\mathcal{C} \leftarrow \mathcal{S}$\tcp*{set the current configuration to starting configuration.}
            $M^* \leftarrow \emptyset$ \tcp*{Initialise empty series.}
    	\While{$\mathcal{C}\neq \mathcal{T}$ \AND $c<0$ \AND $s(\mathcal{C})$ is not undefined}
             {
                $M^* \leftarrow M^*\cup \{\mathcal{C},c\}$\;
                $\mathcal{C} \leftarrow s(\mathcal{C})$ \tcp*{Move to successor.}
                $c \leftarrow c + 1$\;
             }
             \If{$c>o$ $\OR $ $s(\mathcal{C})$ is undefined}
             {
                $W = $\textbf{False}\;
             }
             \textbf{Else}, $W = $\textbf{True}\;
             \If{$W = $\textbf{False}}
             {
                \For{\{$\mathcal{C},c$\} $\in M$}
                {
                    $w(\mathcal{C}) \leftarrow (1-\alpha)*w(\mathcal{C})*(0.95)^c$\;
                }
             }
             \If{$W = $\textbf{True}}
             {
                \For{\{$\mathcal{C},c$\} $\in M$}
                {
                    $w(\mathcal{C}) \leftarrow (1-\alpha)^c*w(C) + \alpha * (0.95)^c$\;
                }
                $M \leftarrow M^*$ \tcp*{The new path is shorter than the earlier optimal one; record it.}
                $o \leftarrow c$\;
             }
        }
	\caption{RL search}\label{Algo:RL}
\end{algorithm}
\newpage
\section{EA performance table} \label{appendix:EATable}
\begin{table}[H]
    \centering
    \caption{\textbf{Summary of EA's performance on higher-dimensional sliding puzzle for $d=3, 4$ and $d=5$}. Face dimensions $k$ vary from $1$ to $d-1$ for all dimensions. The EA is launched $150$ times for each puzzle and face dimension. The success rate is the number of times our EA reaches the target configuration out of the $150$ runs. The difficulty levels range from easy (0) to hard (the most significant integer). The last columns are the minimum, maximum and median number of moves to reach the target configuration $T$. Overall, the EA solves the puzzle after an average of $1000$ generation with a minimum of $4$ moves for $k=1$ and a maximum of $83$ moves. For the puzzle $d=4$, $l=11$, and $k=3$, our EA failed most often to reach the target configuration for different levels after 1000 generations.}
    \vspace{0.2cm}
    \begin{tabular}{cp{2.5cm}p{2cm}ccccc}
         \toprule Dimension ($d$)&Number of uncoloured vertices ($l$)& Face dimension ($k$)& Level& Success rate& Min & Max & Median\\
        \midrule \multirow{6}{0em}{$3$}& \multirow{6}{0em}{$4$}&\multirow{3}{0em}{$1$} & 0& $100\%$& 4 &$10$ & $4$\\
        \cline{4-8}
        & & & 1& $100\%$& $6$ &$10$ & $6$\\
        \cline{4-8}
        & & & 2& $100\%$& $10$ &$18$ & $12$\\ 
        \cline{3-8}
        & &\multirow{3}{0em}{$2$} & 0&$100\%$&$6$&$20$ & $9$\\
        \cline{4-8}
        & & & 1& $100\%$& $7$ &$18$ & $10$\\ 
        \cline{4-8}
        & & & 2& $100\%$& $9$ &$25$ & $14$\\ 
        \midrule \multirow{12}{0em}{$4$}& \multirow{12}{1em}{$11$}&\multirow{4}{0em}{$1$} & 0& $100\%$& $4$& $14$& $4$\\
        \cline{4-8}
        & & & 1& $100\%$& $8$& $60$& $14$\\
        \cline{4-8}
        & & & 2& $100\%$& $6$& $16$& $8$\\
        \cline{4-8}
        & & & 3& $100\%$& $8$& $28$& $12$\\
         \cline{4-8}
        & & & 4& $100\%$& $12$&$28$&$16$ \\
        \cline{3-8}
        & &\multirow{4}{0em}{$2$} & 0& $100\%$&$6$& $57$& $12$\\ \cline{4-8}
       & & & 1&$100\%$ & $10$ & $50$ & $16$\\
       \cline{4-8}
        & & & 2&$100\%$ & $6$ & $13$ & $8$\\
        \cline{4-8}
        & & & 3&$100\%$ & $8$ & $24$ & $14$\\
         \cline{4-8}
        & & & 4& $100\%$&$8$&$22$&$14$ \\
         \cline{3-8}
        & &\multirow{4}{0em}{$3$} & 0& $2\%$&$12$& $15$& $13$\\ \cline{4-8}
       & & & 1&$1.3\%$ & $21$ & $83$ & $52$\\
       \cline{4-8}
        & & & 2&$0.06\%$ & $24$ & $24$ & $24$\\
        \cline{4-8}
        & & & 3&$0.06\%$ & $19$ & $19$ & $19$\\
         \cline{4-8}
        & & & 4& $4.6\%$& $20$&$36$&$26$\\
       \midrule \multirow{5}{0em}{$5$}& \multirow{5}{1em}{$26$}&1& 0& $100\%$& $8$& $22$& $12$\\
        \cline{3-8}
        & &$2$ & $0$&$100\%$& $4$& $17$ &$9$\\
        \cline{3-8}
        & & $3$& $0$&$100\%$& $4$& $14$ &$7$\\
        \cline{3-8}
        & & $4$& $0$& $0\%$& -& - & -\\
        \bottomrule
    \end{tabular}
    \label{tab:sum_performance}
\end{table}
\newpage

\section{RL performance table}\label{appendix:RLTable}
\begin{table}[H]
    \centering
    \caption{\textbf{Summary of RL's performance on higher-dimensional sliding puzzle for $d=3, 4$ and $d=5$}. Face dimensions $k$ vary from $1$ to $d-1$ for both dimensions $d=3$ and $d=4$. The RL is launched $150$ times for each puzzle and face dimension. The success rate is the number of times RL reaches the target configuration out of the 150 runs. The difficulty levels range from easy (0) to hard (the most significant integer). The last columns are the minimum, maximum and median number of moves to reach the target configuration $T$. Overall, the RL solves the puzzle after, on average, $1000$ iterations with a minimum of $4$ moves for $k=1$ and a maximum of $316$ moves for $k=4$. }
    \vspace{0.2cm}
    \begin{tabular}{cp{2.5cm}p{2cm}ccccc}
         \toprule Dimension ($d$)&Number of uncoloured vertice ($l$)& Face dimension ($k$)& Level& Success rate& Min & Max & Median\\
        \midrule \multirow{6}{0em}{$3$}& \multirow{6}{0em}{$4$}&\multirow{3}{0em}{$1$} & 0& $100\%$& $4$&$4$ & $4$\\
        \cline{4-8}
        & & & 1& $100\%$& $6$ &$6$ & $6$\\
        \cline{4-8}
        & & & 2& $100\%$& $10$ &$10$ & $10$\\ 
        \cline{3-8}
        & &\multirow{3}{0em}{$2$} & 0&$100\%$&$6$&$7$ & $6$\\
        \cline{4-8}
        & & & 1& $100\%$& $7$ &$9$ & $7$\\ 
        \cline{4-8}
        & & & 2& $100\%$& $9$ &$12$ & $10$\\ 
        \midrule \multirow{12}{0em}{$4$}& \multirow{12}{1em}{$11$}&\multirow{5}{0em}{$1$} & 0& $100\%$& $4$& $4$& $4$\\
        \cline{4-8}
        & & & 1& $100\%$& $8$& $60$& $14$\\
        \cline{4-8}
        & & & 2& $100\%$& $6$& $16$& $8$\\
        \cline{4-8}
        & & & 3& $100\%$& $8$& $28$& $12$\\
        \cline{4-8}
        & & & 4&$100\%$&$12$&$28$& $16$\\
        \cline{3-8}
        & &\multirow{4}{0em}{$2$} & 0& $100\%$&$6$ & $57$ & $12$\\ \cline{4-8}
       & & & 1&$100\%$ & $10$ & $50$ & $16$\\
       \cline{4-8}
        & & & 2&$100\%$ & $6$ & $13$ & $8$\\
        \cline{4-8}
        & & & 3&$100\%$ & $8$ & $24$ & $14$\\
         \cline{4-8}
        & & & 4&$100\%$&$8$&$22$& $14$\\
         \cline{3-8}
        & &\multirow{4}{0em}{$3$} & 0& $100\%$&$8$& $23$& $14$\\ \cline{4-8}
       & & & 1&$100\%$ & $9$ & $59$ & $24$\\
       \cline{4-8}
        & & & 2&$100\%$ & $11$ & $66$ & $28$\\
        \cline{4-8}
        & & & 3&$100\%$ & $17$ & $105$ & $49$\\
         \cline{4-8}
        & & & 4&$100\%$&$12$&$106$&$41.5$\\
       \midrule \multirow{5}{0em}{$5$}& \multirow{5}{1em}{$26$}&1& $0$& $100\%$& $8$& $10$& $8$ \\
        \cline{3-8}
        & & 2& $0$& $100\%$& $4$& $6$& $5$\\
        \cline{3-8}
        & & 3& $0$& $100\%$& $4$& $6$& $4$\\
        \cline{3-8}
        & & 4& $0$& $100\%$& $62$& $316$& $161$\\
        \hline
    \end{tabular}
    \label{tab:sum_performance_RL}
\end{table}

\end{document}